\documentclass{article}

\usepackage{microtype}
\usepackage[colorlinks,citecolor=blue]{hyperref} 
\usepackage{graphicx}
\usepackage{subcaption}
\usepackage{booktabs} %
\usepackage{verbatim}
\usepackage{commands}
\usepackage{wrapfig}
\usepackage{natbib}

\newcommand{\fedssgd}{\texttt{FetchSGD}}

\newcommand{\fedavg}{\texttt{FedAvg}}

\usepackage[accepted]{icml2020}

\icmltitlerunning{FetchSGD: Communication-Efficient Federated Learning with Sketching}

\begin{document}

\twocolumn[
\icmltitle{FetchSGD: Communication-Efficient Federated Learning with Sketching}

\begin{icmlauthorlist}
\icmlauthor{Daniel Rothchild$^*$}{cal}
\icmlauthor{Ashwinee Panda$^*$}{cal} 
\icmlauthor{Enayat Ullah}{jhu}
\icmlauthor{Nikita Ivkin}{amazon}
\icmlauthor{Ion Stoica}{cal}
\icmlauthor{Vladimir Braverman}{jhu}
\icmlauthor{Joseph Gonzalez}{cal}
\icmlauthor{Raman Arora}{jhu}
\end{icmlauthorlist}

\icmlaffiliation{cal}{University of California, Berkeley, California, USA}
\icmlaffiliation{jhu}{Johns Hopkins University, Baltimore, Maryland}
\icmlaffiliation{amazon}{Amazon}

\icmlcorrespondingauthor{Daniel Rothchild}{drothchild@berkeley.edu}

\icmlkeywords{Machine Learning, ICML, sketching, count sketch, federated learning}

\vskip 0.3in
]

\printAffiliationsAndNotice{\icmlEqualContribution} %

\begin{abstract}
Existing approaches to federated learning suffer from a communication bottleneck as well as convergence issues due to sparse client participation.
In this paper we introduce a novel algorithm, called \fedssgd{}, to overcome these challenges.
\fedssgd{} compresses model updates using a Count Sketch, and then takes advantage of the mergeability of sketches to combine model updates from many workers.
A key insight in the design of \fedssgd{} is that, because the Count Sketch is linear, momentum and error accumulation can both be carried out within the sketch.
This allows the algorithm to move momentum and error accumulation from clients to the central aggregator, overcoming the challenges of sparse client participation while still achieving high compression rates and good convergence.
We prove that \fedssgd{} has favorable convergence guarantees, and we demonstrate its empirical effectiveness by training two residual networks and a transformer model.
\end{abstract}

\vspace*{-20pt}
\section{Introduction}
Federated learning has recently emerged as an important setting for training machine learning models.
In the federated setting, training data is distributed across a large number of edge devices, such as consumer smartphones, personal computers, or smart home devices.
These devices have data that is useful for training a variety of models -- for text prediction, speech modeling, facial recognition, document identification, and other tasks~\citep{edgedata, health, kws, phonedata}. 
However, data privacy, liability, or regulatory concerns may make it difficult to move this data to the cloud for training \citep{gdpr}.
Even without these concerns, training machine learning models in the cloud can be expensive, and an effective way to train the same models on the edge has the potential to eliminate this expense.%

When training machine learning models in the federated setting, participating clients do not send their local data to a central server; instead, a central aggregator coordinates an optimization procedure among the clients.
At each iteration of this procedure, clients compute gradient-based updates to the current model using their local data, and they communicate only these updates to a central aggregator.

A number of challenges arise when training models in the federated setting.
Active areas of research in federated learning include solving systems challenges, such as handling stragglers and unreliable network connections \citep{secureagg, resourceconstrained}, tolerating adversaries \citep{backdoor, adversariallens}, and ensuring privacy of user data \citep{dpfl, homoencryption}.
In this work we address a different challenge, namely that of training high-quality models under the constraints imposed by the federated setting.

There are three main constraints unique to the federated setting that make training high-quality models difficult. 
First, \textbf{communication-efficiency} is a necessity when training on the edge~\citep{edgelearning}, since clients typically connect to the central aggregator over slow connections ($\sim1$Mbps) \citep{mobilewifi}.
Second, \textbf{clients must be stateless,} since it is often the case that no client participates more than once during all of training \citep{survey}.
Third, the \textbf{data collected across clients is typically not independent and identically distributed.}
For example, when training a next-word prediction model on the typing data of smartphone users, clients located in geographically distinct regions generate data from different distributions, but enough commonality exists between the distributions that we may still want to train a single model \citep{gboard, improvinggboard}.

In this paper, we propose a new optimization algorithm for federated learning, called \fedssgd{}, that can train high-quality models under all three of these constraints.
The crux of the algorithm is simple: at each round, clients compute a gradient based on their local data, then compress the gradient using a data structure called a Count Sketch before sending it to the central aggregator.
The aggregator maintains momentum and error accumulation Count Sketches, and the weight update applied at each round is extracted from the error accumulation sketch.
See Figure \ref{fig:overview} for an overview of \fedssgd{}.

\fedssgd{} requires no local state on the clients, and we prove that it is communication efficient, and that it converges in the non-i.i.d. setting for $L$-smooth non-convex functions at rates $\bigO{T^{-1/2}}$ and $\bigO{T^{-1/3}}$ respectively under two alternative assumptions -- the first opaque and the second more intuitive.
Furthermore, even without maintaining any local state, \fedssgd{} can carry out momentum -- a technique that is essential for attaining high accuracy in the non-federated setting -- as if on local gradients before compression \citep{sutskever2013importance}.
Lastly, due to properties of the Count Sketch, \fedssgd{} scales seamlessly to small local datasets, an important regime for federated learning, since user interaction with online services tends to follow a power law distribution, meaning that most users will have relatively little data to contribute \citep{muchnik2013origins}.

We empirically validate our method with two image recognition tasks and one language modeling task.
Using models with between 6 and 125 million parameters, we train on non-i.i.d. datasets that range in size from 50,000 -- 800,000 examples.

\begin{figure*}
    \centering
    \includegraphics[height=5cm]{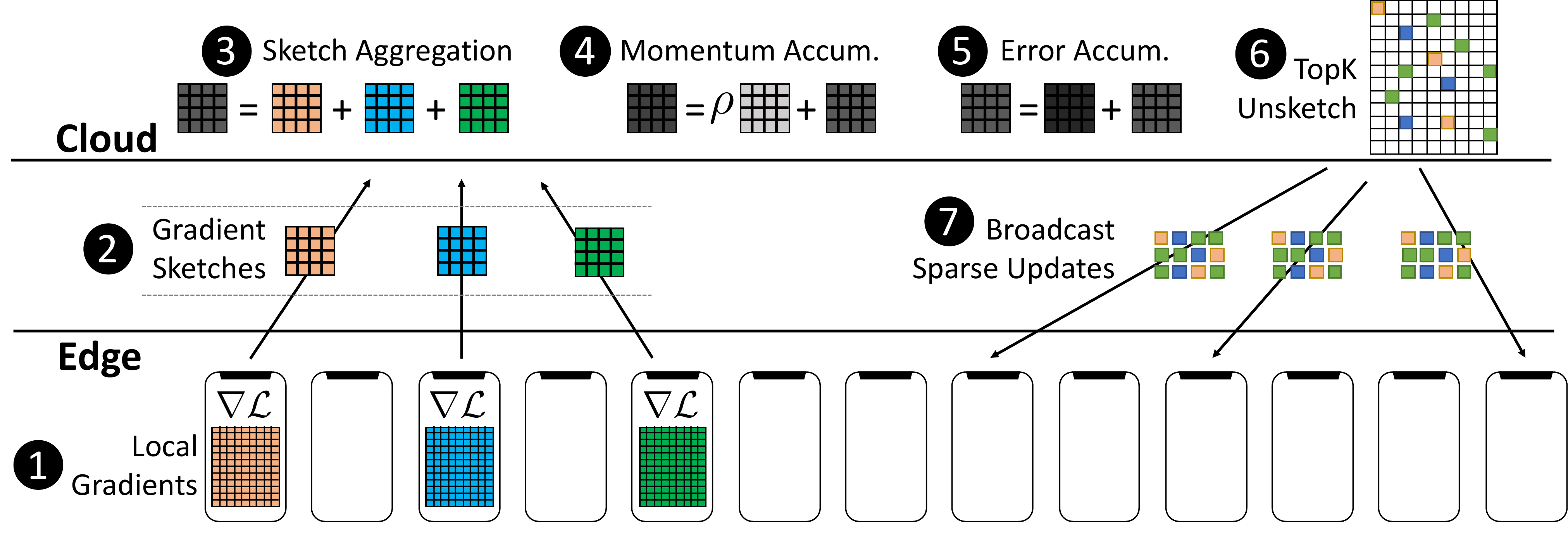}
    \caption{\textbf{Algorithm Overview}.  
    The \fedssgd{} algorithm \textbf{(1)} computes gradients locally, and then send sketches \textbf{(2)} of the gradients to the cloud.  In the cloud,  gradient sketches are aggregated \textbf{(3)}, and then \textbf{(4)} momentum and \textbf{(5)} error accumulation are applied to the sketch.  The approximate top-k values are then \textbf{(6)} extracted and \textbf{(7)} broadcast as sparse updates to devices participating in  next~round.  }
    \vspace{-0.5cm}
    \label{fig:overview}
\end{figure*}
\section{Related Work}
\label{sec:related}

Broadly speaking, there are two optimization strategies that have been proposed to address the constraints of federated learning: Federated Averaging (\fedavg{}) and extensions thereof, and gradient compression methods.
We explore these two strategies in detail in Sections \ref{sec:related_fedavg} and \ref{sec:related_compression}, but as a brief summary, \fedavg{} does not require local state, but it also does not reduce communication from the standpoint of a client that participates once, and it struggles with non-i.i.d. data and small local datasets because it takes many local gradient steps.
Gradient compression methods, on the other hand, can achieve high communication efficiency.
However, it has been shown both theoretically and empirically that these methods must maintain error accumulation vectors on the clients in order to achieve high accuracy. This is ineffective in federated learning, since clients typically participate in optimization only once, so the accumulated error has no chance to be re-introduced \citep{karimireddy2019error}.

\subsection{\fedavg{}}
\label{sec:related_fedavg}
\fedavg{} reduces the total number of bytes transferred during training by carrying out multiple steps of stochastic gradient descent (SGD) locally before sending the aggregate model update back to the aggregator.
This technique, often referred to as local/parallel SGD, has been studied since the early days of distributed model training in the data center \citep{dean2012large}, and is referred to as \fedavg{} when applied to federated learning \citep{fedavg}.
\fedavg{} has been successfully deployed in a number of domains \citep{gboard, li2019privacy}, and is the most commonly used optimization algorithm in the federated setting \citep{improvinggboard}.
In \fedavg{}, every participating client first downloads and trains the global model on their local dataset for a number of epochs using SGD. %
The clients upload the difference between their initial and final model to the parameter server, which averages the local updates weighted according to the magnitude of the corresponding local dataset.

One major advantage of \fedavg{} is that it requires no local state, which is necessary for the common case where clients participate only once in training. 
\fedavg{} is also communication-efficient in that it can reduce the total number of bytes transferred during training while achieving the same overall performance.
However, from an individual client's perspective, there is no communication savings if the client participates in training only once.
Achieving high accuracy on a task often requires using a large model, but clients' network connections may be too slow or unreliable to transmit such a large amount~of~data at once \citep{yang2010mobile}.

Another disadvantage of \fedavg{} is that taking many local steps can lead to degraded convergence on non-i.i.d. data.
Intuitively, taking many local steps of gradient descent on local data that is not representative of the overall data distribution will lead to local over-fitting, which will hinder convergence \citep{scaffold}.
When training a model on non-i.i.d. local datasets, the goal is to minimize the average test error across clients. 
If clients are chosen randomly, SGD naturally has convergence guarantees on non-i.i.d.\ data, since the average test error is an expectation over which clients participate.
However, although \fedavg{} has convergence guarantees for the i.i.d. setting \citep{coopsgd}, these guarantees do not apply directly to the non-i.i.d. setting as they do with SGD.
\citet{noniid} show that \fedavg{}, using $K$ local steps, converges as $\bigO{K/T}$ on non-i.i.d. data for strongly convex smooth functions, with additional assumptions.
In other words, convergence on non-i.i.d. data could slow down as much as proportionally to the number of local steps taken.

Variants of \fedavg{} have been proposed to improve its performance on non-i.i.d. data.
\citet{sahu2018convergence} propose constraining the local gradient update steps in \fedavg{} by penalizing the L2 distance between local models and the current global model.
Under the assumption that every client's loss is minimized wherever the overall loss function is minimized, they recover the convergence rate of SGD.
\citet{scaffold} modify the local updates in \fedavg{} to make them point closer to the consensus gradient direction from all clients.
They achieve good convergence at the cost of making the clients stateful.

\subsection{Gradient Compression}
\label{sec:related_compression}

A limitation of \fedavg{} is that, in each communication round, clients must download an entire model and upload an entire model update.
Because federated clients are typically on slow and unreliable network connections, this requirement makes training large models with \fedavg{} difficult.
Uploading model updates is particularly challenging, since residential Internet connections tend to be asymmetric, with far higher download speeds than upload speeds \citep{goga2012speed}.

An alternative to \fedavg{} that helps address this problem is regular distributed SGD with gradient compression.
It is possible to compress stochastic gradients such that the result is still an unbiased estimate of the true gradient, for example by stochastic quantization \citep{qsgd} or stochastic sparsification \citep{sparsification}.
However, there is a fundamental tradeoff between increasing compression and increasing the variance of the stochastic gradient, which slows convergence.
The requirement that gradients remain unbiased after compression is too stringent, and these methods have had limited empirical success.

Biased gradient compression methods, such as top-$k$ sparsification \citep{dgc} or signSGD \citep{bernstein2018signsgd}, have been more successful in practice.
These methods rely, both in theory and in practice, on the ability to locally accumulate the error introduced by the compression scheme, such that the error can be re-introduced the next time the client participates \citep{karimireddy2019error}.
Unfortunately, carrying out error accumulation requires local client state, which is often infeasible in federated learning.

\subsection{Optimization with Sketching}

This work advances the growing body of research applying sketching techniques to optimization.
\citet{jiang2018sketchml} propose using sketches for gradient compression in data center training.
Their method achieves empirical success when gradients are sparse, but it has no convergence guarantees, and it achieves little compression on dense gradients \citep[\textsection B.3]{jiang2018sketchml}.
The method also does not make use of error accumulation, which more recent work has demonstrated is necessary for biased gradient compression schemes to be successful \citep{karimireddy2019error}.
\citet{ivkin2019communication} also propose using sketches for gradient compression in data center training.
However, their method requires a second round of communication between the clients and the parameter server, after the first round of transmitting compressed gradients completes.
Using a second round is not practical in federated learning, since stragglers would delay completion of the first round, at which point a number of clients that had participated in the first round would no longer be available \citep{secureagg}.
Furthermore, the method in \cite{ivkin2019communication} requires local client state for both momentum and error accumulation, which is not possible in federated learning.
\citet{spring2019compressing} also propose using sketches for distributed optimization.
Their method compresses auxiliary variables such as momentum and per-parameter learning rates, without compressing the gradients themselves.
In contrast, our method compresses the gradients, and it does not require any additional communication at all to carry out momentum.

\citet{strategies} propose using sketched updates to achieve communication efficiency in federated learning.
However, the family of sketches they use differs from the techniques we propose in this paper: they apply a combination of subsampling, quantization and random rotations.

\section{\fedssgd{}}
\subsection{Federated Learning Setup}

Consider a federated learning scenario with $C$ clients.
Let $\cZ$ be the data domain and let $\{\cP_i\}_{i=1}^C$ be $C$ possibly unrelated probability distributions over $\cZ$.
For supervised learning, $\cZ = \cX \times \cY$, where $\cX$ is the feature space and $\cY$ is the label space; for unsupervised learning, $\cZ=\cX$ is the feature space.
The $i^{\text{th}}$ client has $D_i$ samples drawn i.i.d. from the $\cP_i$.
Let $\cW$ be the hypothesis class parametrized by $d$ dimensional vectors.
Let ${\cL:\cW \times \cZ \rightarrow \bbR}$ be a loss function. The goal is to minimize the weighted average $\hat{\mathbb{E}}$ of  client risks:
\vspace*{-5pt}
\begin{align}
    \label{eqn:noniid1}
   f(\w)\! = \hat{\mathbb{E}}f_i(\w)\! = \!
    \frac{1}{\sum_{i=1}^C D_i}\sum_{i=1}^C D_i\!\Eu{\z \sim \cP_i}\cL(\w,\z)
\end{align}

Assuming that all clients have an equal number of data points, this simplifies to the average of client risks:
\vspace*{-2pt}
\begin{align}
\label{eqn:noniid2}
   f(\w) = \hat{\mathbb{E}}f_i(\w) =
    \frac{1}{C}\sum_{i=1}^C \Eu{\z \sim \cP_i}\cL(\w,\z).
\end{align}
\vspace*{-2pt}

For simplicity of presentation, we consider this unweighted average (eqn. \ref{eqn:noniid2}), but our theoretical results directly extend to the the more general setting (eqn. \ref{eqn:noniid1}).

In federated learning, a central aggregator coordinates an iterative optimization procedure to minimize $f$ with respect to the model parameters $\w$.
In every iteration, the aggregator chooses $W$ clients uniformly at random,\footnote{In practice, the clients may not be chosen randomly, since often only devices that are on wifi, charging, and idle are allowed to participate.} and these clients download the current model, determine how to best update the model based on their local data, and upload a model update to the aggregator.
The aggregator then combines these model updates to update the model for the next iteration.
Different federated optimization algorithms use different model updates and different aggregation schemes to combine these updates.

\vspace*{-5pt}
\subsection {Algorithm}
\label{subsec:alg}
\vspace*{-2pt}

At each iteration in \fedssgd{}, the $i^{\text{th}}$ participating client computes a stochastic gradient $\g^t_{i}$ using a batch of (or all of) its local data, then compresses $\g^t_i$ using a data structure called a Count Sketch. %
Each client then sends the sketch $\cS(\g^t_i)$ to the aggregator as its model update.

A Count Sketch is a randomized data structure that can compress a vector by randomly projecting it several times to lower dimensional spaces, such that high-magnitude elements can later be approximately recovered.
We provide more details on the Count Sketch in Appendix~\ref{appendix:countsketch}, but here we treat it simply as a compression operator $\cS(\cdot)$, with the special property that it is linear:
\[\cS(\g_1+\g_2)=\cS(\g_1)+\cS(\g_2).\]
Using linearity, the server can exactly compute the sketch of the true minibatch gradient $\g^t=\sum_i \g^t_i$ given only the $\cS(\g^t_i)$:
\[\sum_i \cS(\g^t_i) = \cS\left(\sum_i \g^t_i\right) =  \cS(\g^t).\]
Another useful property of the Count Sketch is that, for a sketching operator $\cS(\cdot)$, there is a corresponding decompression operator $\cU(\cdot)$ that returns an unbiased estimate of the original vector, such that the high-magnitude elements of the vector are approximated well (see Appendix \ref{appendix:countsketch} for details):
\[\text{Top-k}(\cU(\cS(\g))) \approx \text{Top-k}(\g).\]

Briefly, $\cU(\cdot)$ approximately ``undoes'' the projections computed by $\cS(\cdot)$, and then uses these reconstructions to estimate the original vector. 
See Appendix \ref{appendix:countsketch} for more details.

With the $\cS(\g_i^t)$ in hand, the central aggregator could update the global model with $\text{Top-k}\left(\cU(\sum_i\cS(\g_i^t))\right) \approx \text{Top-k}\left(\g^t\right)$.
However, $\text{Top-k}(\g^t)$ is not an unbiased estimate of $\g^t$, so the normal convergence of SGD does not apply.
Fortunately, \citet{karimireddy2019error} show that biased gradient compression methods can converge if they accumulate the error incurred by the biased gradient compression operator and re-introduce the error later in optimization.
In \fedssgd{}, the bias is introduced by $\text{Top-k}$ rather than by $\cS(\cdot)$, so the aggregator, instead of the clients, can accumulate the error, and it can do so into a zero-initialized sketch $S_e$ instead of into a gradient-like vector:

\vspace*{-8pt}
{\begin{small}
\begin{align}
    \bS^t &= \frac{1}{W}\sum_{i=1}^W\cS(\g^t_{i})  \nonumber \\ %
    \Delta^t &= \text{Top-k}(\cU(\eta \bS^{t} + \bS_e^t)))\nonumber \\
    \bS_e^{t+1} &=  \eta \bS^{t} + \bS_e^{t} - \cS(\Delta^t) \nonumber\\
    \w^{t+1} &= \w^{t} - \Delta^t, \nonumber 
\end{align}
\end{small}}
where $\eta$ is the learning rate and $\Delta^t\in \mathbb{R}^d$ is $k$-sparse. 

In contrast, other biased gradient compression methods introduce bias on the clients when compressing the gradients, so the clients themselves must maintain individual error accumulation vectors.
This becomes a problem in federated learning, where clients may participate only once, giving the error no chance to be reintroduced in a later round.

Viewed another way, because $\cS(\cdot)$ is linear, and because error accumulation consists only of linear operations, carrying out error accumulation on the server within $S_e$ is equivalent to carrying out error accumulation on each client, and uploading sketches of the result to the server.
(Computing the model update from the accumulated error is not linear, but only the server does this, whether the error is accumulated on the clients or on the server.)
Taking this a step further, we note that momentum also consists of only linear operations, and so momentum can be equivalently carried out on the clients or on the server. Extending the above equations with momentum yields
\vspace*{-25pt} %

{\begin{small}
\begin{align}
    \bS^t &= \frac{1}{W}\sum_{i=1}^W\cS(\g^t_{i})  \nonumber \\ %
    \bS_u^{t+1} &= \rho \bS_u^{t} + \bS^t  \nonumber \\
    \Delta &= \text{Top-k}(\cU(\eta \bS_u^{t+1} + \bS_e^t)))\nonumber \\
    \bS_e^{t+1} &=  \eta \bS_u^{t+1} + \bS_e^{t} - \cS(\Delta) \nonumber\\
    \w^{t+1} &= \w^{t} - \Delta. \nonumber 
\end{align}
\end{small}}

\vspace*{-20pt} %

\fedssgd{} is presented in full in Algorithm \ref{alg:fedsketchedsgd}.

\begin{algorithm}
\caption{\fedssgd{}}
\label{alg:fedsketchedsgd}
    \begin{algorithmic}[1]
    {\begin{small}
	\REQUIRE number of model weights to update each round $k$
	\REQUIRE learning rate $\eta$
	\REQUIRE number of timesteps $T$
	\REQUIRE momentum parameter $\rho$, local batch size $\ell$
	\REQUIRE Number of clients selected per round $W$
	\REQUIRE Sketching and unsketching functions $\mathcal{S}$, $\mathcal{U}$
    \STATE Initialize $\bS_u^0$ and $\bS_e^0$ to zero sketches
    \STATE Initialize $\w^0$ using the same random seed on the clients and aggregator 
	\FOR{$t = 1,2,\cdots T$}
	    \STATE Randomly select $W$ clients $c_1,\ldots c_W$
	    \LOOP[In parallel on clients $\bc{c_i}_{i=1}^W$]
	        \STATE Download (possibly sparse) new model weights $\w^t-\w^0$ %
    		\STATE Compute stochastic gradient $\g_i^t$ on batch $B_i$ of size $\ell$:\\ 
    		$\g^t_i = \frac{1}{\ell}\sum_{j=1}^l\nabla_\w \mathcal{L}(\w^t, \z_j)$
    		\STATE Sketch $\g_i^t$: $\bS_i^t = \mathcal{S}(\g_i^t)$ and send it to the Aggregator 
    	\ENDLOOP
		\STATE Aggregate sketches $\bS^t = \frac{1}{W}\sum_{i=1}^W \bS_i^t$ 
		\STATE Momentum: $\bS_u^t = \rho \bS_u^{t-1} + \bS^t$ %
 		\STATE Error feedback: $\bS_e^t = \eta\bS_u^t + \bS_e^t$ 
 		\STATE Unsketch: $\Delta^t = \text{Top-k}(\mathcal{U}(\bS_e^t))$
		\STATE Error accumulation: $\bS_e^{t+1} = \bS_e^t -S(\Delta^t)$

	    \STATE Update $ \w^{t+1} = \w^{t} - \Delta^t$ 
    \ENDFOR 
	\ENSURE  $\bc{\w^t}_{t=1}^T$
	\end{small}}
	\end{algorithmic}
\end{algorithm}

\vspace*{-10pt} %
\section{Theory}
\label{sec:theory}

This section presents convergence guarantees for \fedssgd{}.
First, Section \ref{sec:contraction} gives the convergence of \fedssgd{} when making a strong and opaque assumption about the sequence of gradients.
Section \ref{sec:sliding} instead makes a more interpretable assumption about the gradients, and arrives at a weaker convergence guarantee.
\subsection{Scenario 1: Contraction Holds}
\label{sec:contraction}
To show that compressed SGD converges when using some biased gradient compression operator $\cC(\cdot)$, existing methods
\citep{karimireddy2019error,zheng2019communication,ivkin2019communication} appeal to
\citet{stich2018sparsified}, %
who show that compressed SGD converges when $\cC$ is a $\tau$-contraction:
\[\norm{\cC(\x)-\x}\leq(1-\tau)\norm{\x}\]
\citet{ivkin2019communication} show that it is possible to satisfy this contraction property using Count Sketches to compress gradients.
However, their compression method includes a second round of communication: if there are no high-magnitude elements in $\e^t$, as computed from $\cS(\e^t)$, the server can query clients for random entries of $\e^t$.
On the other hand, \fedssgd{} never computes the $\e^t_i$, or $\e^t$, so this second round of communication is not possible, and the analysis of \citet{ivkin2019communication} does not apply.
In this section, we assume that the updates have heavy hitters, which ensures that the contraction property holds along the optimization~path.

\begin{assumption}[Scenario 1]
\label{ass:contraction}
Let $\bc{\w_t}_{t=1}^T$ be the sequence of models generated by \fedssgd{}. Fixing this model sequence, let $\{\u^t\}_{t=1}^T$ and $\{\e^t\}_{t=1}^T$ be the momentum and error accumulation vectors generated using this model sequence, had we not used sketching for gradient compression (i.e. if $\cS$ and $\cU$ are identity maps). There exists a constant $0<\tau< 1$ such that for any $t \in [T]$, the quantity $q^t:=\eta(\rho \u^{t-1}+ \g^{t-1})+\e^{t-1}$ has at least one coordinate $i$ s.t. $(q^t_i)^2 \ge \tau \norm{q^t_i}^2$.

\end{assumption}

\begin{theorem}[Scenario 1]
\label{thm:convergence_assumption}
Let $f$ be an $L$-smooth \footnote{A differentiable function $f$ is $L$-smooth if $\norm{\nabla f(\x)-\nabla f(\y)} \leq L\norm{\x-\y} \ \forall \ \x,\y \in \text{dom}(f)$.} non-convex function and let the norm of stochastic gradients of $f$ be upper bounded by $G$.
Under Assumption \ref{ass:contraction}, \fedssgd{}, with step size $\eta=\frac{1-\rho}{2L\sqrt{T}}$, in $T$ iterations, returns $\{\w^t\}_{t=1}^T$, such that, with probability at least $1-\delta$ over the sketching randomness:
{%
\begin{enumerate}
    \item 
    \begin{small} ${\underset{t=1 \cdots T}{\min}\E{\norm{ \nabla f( \w^t)}^2}~\leq~\frac{4L(f(\w^0)-f^*)~+~G^2)}{\sqrt{T}} + \frac{2(1+\tau)^2G^2}{(1-\tau)\tau^2 T}}.$
    \end{small}
    \item The sketch uploaded from each participating client to the parameter server is $\bigO{\log{dT/\delta}/\tau}$  bytes per round.
    \end{enumerate}
}
\vspace{-0.3cm}
\end{theorem}
The expectation in part 1 of the theorem is over the randomness of sampling minibatches.
For large $T$, the first term dominates, so the convergence rate in Theorem \ref{thm:convergence_assumption} matches that of uncompressed SGD.

Intuitively, Assumption \ref{ass:contraction} states that, at each time step, the descent direction -- \textit{i.e.}, the scaled negative gradient, including momentum -- and the error accumulation vector must point in sufficiently the same direction.
This assumption is rather opaque, since it involves all of the gradient, momentum, and error accumulation vectors, and it is not immediately obvious that we should expect it to hold.
To remedy this, the next section analyzes \fedssgd{} under a simpler assumption that involves only the gradients.
Note that this is still an assumption on the algorithmic path, but it presents a clearer understanding. 

\vspace*{-4pt}
\subsection{Scenario 2: Sliding Window Heavy Hitters}
\label{sec:sliding}

Gradients taken along the optimization path have been observed to contain heavy coordinates \citep{shi2019understanding,li2019privacy}.
However, it would be overly optimistic to assume that \emph{all} gradients contain heavy coordinates, since this might not be the case in some flat regions of parameter space.
Instead, we introduce a much milder assumption: namely that there exist heavy coordinates in a sliding sum of gradient vectors:

\begin{definition}\label{defn:sliding_heavy}[$(I,\tau)$-sliding heavy\footnote{Technically, this definition is also parametrized by $\delta$ and $\beta$. However, in the interest of brevity, we use the simpler term ``$(I, \tau)$-sliding heavy'' throughout the manuscript. Note that $\delta$ in Theorem \ref{thm:main} refers to the same $\delta$ as in Definition \ref{defn:sliding_heavy}.}
]\; \\ \;
A stochastic process $\bc{\g^t}_{t \in \bbN}$ is $(I,\tau)$-sliding heavy if with probability at least $1-\delta$, at every iteration $t$, the gradient vector $\g^t$ can be decomposed as ${\g^t = \g^t_N + \g^t_S}$, 
where $\g^t_S$ is ``signal'' and $\g^t_N$ is ``noise'' with the following properties: 
\vspace*{-10pt} 
\begin{CompactEnumerate}
  \item \textbf{[Signal]} For every non-zero coordinate $j$ of vector~$\g^t_S$, 
  $\exists t_1,t_2$ with ${t_1 \le t \le t_2},\;t_2 - t_1 \le I\;\!\!$ s.t.$|\sum_{t_1}^{t_2}{\g^{t}_j}|\!\!> \!\!\tau\|\sum_{t_1}^{t_2}{\g^{t}}\|$.
  \item \textbf{[Noise]} $\g^{t}_N$ is mean zero, symmetric and when normalized by its norm, its second moment bounded as $\E{{\frac{\norm{\g^{t}_N}^2}{\norm{\g^t}^2}}}\leq \beta $.
\end{CompactEnumerate}
\end{definition}

\vspace*{-5pt}
Intuitively, this definition states that, if we sum up to $I$ consecutive gradients, every coordinate in the result will either be an $\tau$-heavy hitter, or will be drawn from some mean-zero symmetric noise.
When $I=1$, part 1 of the definition reduces to the assumption that gradients always contain heavy coordinates.
Our assumption for general, constant $I$ is significantly weaker, as it requires the gradients to have heavy coordinates in a sequence of $I$ iterations rather than in every iteration.
The existence of heavy coordinates spread across consecutive updates helps to explains the success of error feedback techniques, which extract signal from a sequence of gradients that may be indistinguishable from noise in any one iteration.
Note that both the signal and the noise scale with the norm of the gradient, so both adjust accordingly as gradients become smaller later in optimization.

Under this definition, we can use Count Sketches to capture the signal, since Count Sketches can approximate heavy hitters. 
Because the signal is spread over sliding windows of size $I$, we need a sliding window error accumulation scheme to ensure that we capture whatever signal is present.
Vanilla error accumulation is not sufficient to show convergence, since vanilla error accumulation sums up \textit{all} prior gradients, so signal that is present only in a sum of $I$ consecutive gradients (but not in $I+1$, or $I+2$, etc.) will not be captured with vanilla error accumulation.
Instead, we can use a sliding window error accumulation scheme, which can capture any signal that is spread over a sequence of at most $I$ gradients.
One simple way to accomplish this is to maintain $I$ error accumulation Count Sketches, as shown in Figure \ref{swpic} for $I=4$.
Each sketch accumulates new gradients every iteration, and beginning at offset iterations, each sketch is zeroed out every $I$ iterations before continuing to accumulate gradients (this happens after line 15 of Algorithm \ref{alg:fedsketchedsgd}).
Under this scheme, at every iteration there is a sketch available that contains the sketched sum of the prior $I'$ gradients, for all $I'\leq I$.
We prove convergence in Theorem \ref{thm:main} when using this sort of sliding window error accumulation scheme.

In practice, it is too expensive to maintain $I$ error accumulation sketches.
Fortunately, this ``sliding window'' problem is well studied 
~\citep{datar2002maintaining,braverman2007smooth,braverman2014catch,braverman2015zero,braverman2018nearly,braverman2018near}, 
and it is possible to identify heavy hitters 
with only $\log I$ error accumulation sketches.
Additional details on sliding window Count Sketch are in  Appendix~\ref{appendix:sliding_window}.
Although we use a sliding window error accumulation scheme to prove convergence, in all experiments we use a single error accumulation sketch, since we find that doing so still leads to good convergence.

\begin{figure}
\centering
\includegraphics[width=170px]{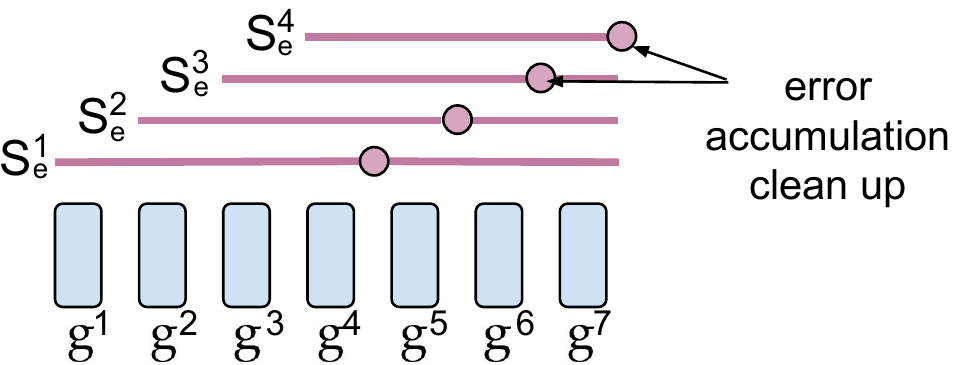}
\caption{Sliding window error accumulation}
\label{swpic}
\end{figure}

\begin{assumption}[Scenario 2]
\label{ass:sliding_heavy}
The sequence of gradients encountered during optimization form an $(I,\tau)$-sliding heavy stochastic process.
\end{assumption}
\begin{theorem}[Scenario 2]
\label{thm:main}
Let $f$ be an $L$-smooth non-convex function and let $\g_i$ denote stochastic gradients of $f_i$ such that
$\norm{\g_i}^2 \leq G^2$.
Under Assumption \ref{ass:sliding_heavy}, \fedssgd{}, using a sketch size $\Theta\left(\frac{\log{dT/\delta}}{\tau^2}\right)$, with step size $\eta=\frac{1}{G\sqrt{L}T^{2/3}}$ and $\rho=0$ (no momentum), in $T$ iterations, with probability at least $1-2\delta$, returns $\{\w^t\}_{t=1}^T$ such that
\vspace*{-5pt}
\begin{CompactEnumerate}
\item {\small $ \underset{t=1\cdots T}{\mathop{\min}}\E{\norm{\nabla f(\w^t)}^2}\!\leq \!
 \frac{ G\sqrt{L}(f(\w^0) -f^*) + 2(2-\tau)}{T^{1/3}} \!\!+\!\frac{G \sqrt{L}}{T^{2/3}}\!+ \!\frac{2I^2}{T^{4/3}}$}
\item The sketch uploaded from each participating client to the parameter server is $\Theta\left(\frac{\log{dT/\delta}}{\tau^2}\right)$ bytes per round.
\end{CompactEnumerate}

\end{theorem}
As in Theorem \ref{thm:convergence_assumption}, the expectation in part 1 of the theorem is over the randomness of sampling minibatches.

\textbf{Remarks:}
\begin{CompactEnumerate}
\vspace*{-5pt}
    \item These guarantees are for the non-i.i.d.\ setting -- i.e. $f$ is the average risk with respect to potentially unrelated distributions (see eqn. \ref{eqn:noniid2}).
    \item The convergence rates bound the objective gradient norm rather than the objective itself.
    \item The convergence rate in Theorem \ref{thm:convergence_assumption} matches that of uncompressed SGD, while the rate in Theorem \ref{thm:main} is worse.
    \item The proof uses the virtual sequence idea of \citet{stich2018sparsified}, and can be generalized to other class of functions like smooth, (strongly) convex etc. by careful averaging (proof in Appendix \ref{appendix:proof}).
\end{CompactEnumerate}

\vspace*{-13pt}
\section{Evaluation}
\vspace*{-2pt}
We implement and compare \fedssgd{}, gradient sparsification (local top-$k$), and \fedavg{} using PyTorch \citep{paszke2017automatic}.\footnote{Code available at \url{https://github.com/kiddyboots216/CommEfficient}. Git commit at the time of camera-ready: 833ca44.}
In contrast to our theoretical assumptions, we use neural networks with ReLU activations, whose loss surfaces are not $L$-smooth.
In addition, although Theorem \ref{thm:main} uses a sliding window Count Sketch for error accumulation, in practice we use a vanilla Count Sketch.
Lastly, we use non-zero momentum, which Theorem \ref{thm:convergence_assumption} allows but Theorem \ref{thm:main} does not.
We also make two changes to Algorithm \ref{alg:fedsketchedsgd}.
For all methods, we employ momentum factor masking \citep{dgc}.
And on line 14 of Algorithm \ref{alg:fedsketchedsgd}, we zero out the nonzero coordinates of $\cS(\Delta^t)$ in $S_e^t$ instead of subtracting $\cS(\Delta^t)$; empirically, doing so stabilizes the optimization.

We focus our experiments on the regime of small local datasets and non-i.i.d. data, since we view this as both an important and relatively unsolved regime in federated learning.
Gradient sparsification methods, which sum together the local top-$k$ gradient elements from each worker, do a worse job approximating the true top-$k$ of the global gradient as local datasets get smaller and more unlike each other.
And taking many steps on each client's local data, which is how \fedavg{} achieves communication efficiency, is unproductive since it leads to immediate local overfitting.
However, real-world users tend to generate data with sizes that follow a power law distribution \citep{goyal2017accurate}, so most users will have relatively small local datasets.
Real data in the federated setting is also typically non-i.i.d.

\fedssgd{} has a key advantage over prior methods in this regime because our compression operator is linear.
Small local datasets pose no difficulties, since executing a step using only a single client with $N$ data points is equivalent to executing a step using $N$ clients, each of which has only a single data point.
By the same argument, issues arising from non-i.i.d. data are partially mitigated by random client selection, since combining the data of participating clients leads to a more representative sample of the full data distribution.

\begin{figure*}[t]
\centering
\begin{subfigure}[b]{0.49\textwidth}
    \includegraphics[width=0.90\textwidth]{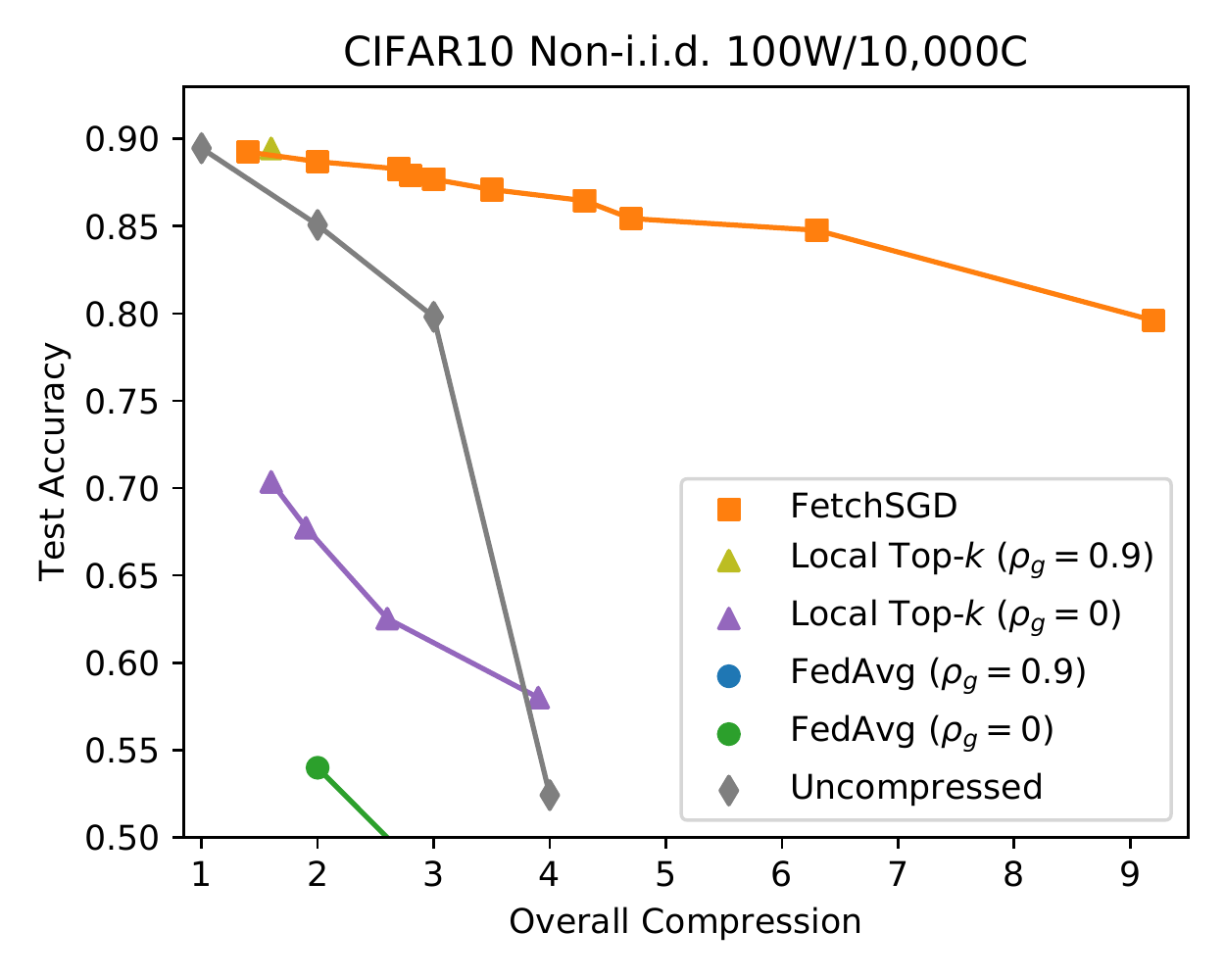}
    \label{fig:cifar10}
    \vspace{-.5cm} %
\end{subfigure}
\begin{subfigure}[b]{0.49\textwidth}
    \includegraphics[width=0.90\textwidth]{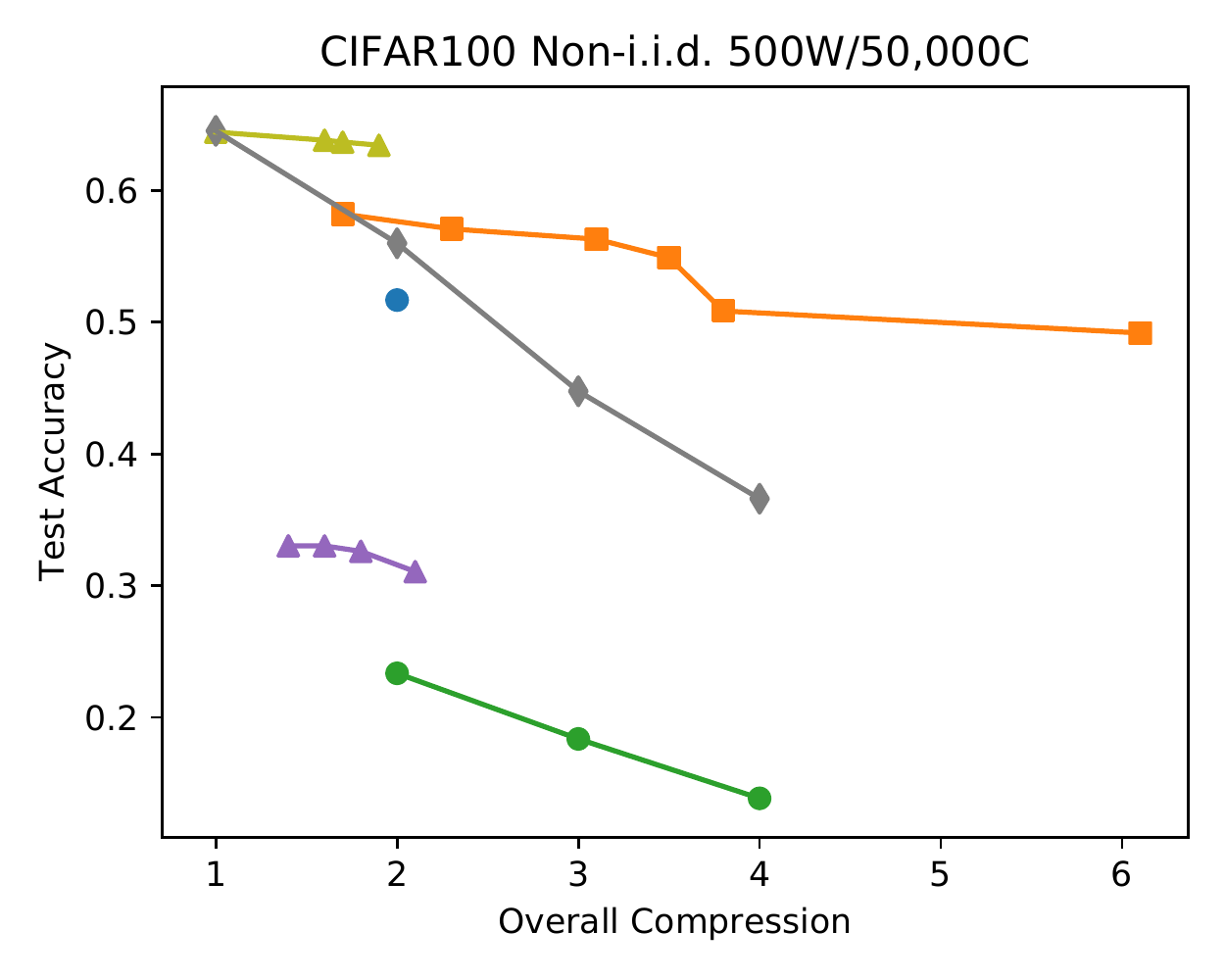}
    \label{fig:cifar100}
    \vspace{-.5cm} %
\end{subfigure}
\caption{Test accuracy achieved on CIFAR10 (left) and CIFAR100 (right). ``Uncompressed'' refers to runs that attain compression by simply running for fewer epochs. \fedssgd{} outperforms all methods, especially at higher compression. Many \fedavg{} and local top-$k$ runs are excluded from the plot because they failed to converge or achieved very low accuracy.}
\label{fig:cifar}
\end{figure*}

For each method, we report the compression achieved relative to uncompressed SGD in terms of total bytes uploaded and downloaded.\footnote{We only count non-zero weight updates when computing how many bytes are transmitted. This makes the unrealistic assumption that we have a zero-overhead sparse vector encoding scheme.}
One important consideration not captured in these numbers is that in \fedavg{}, clients must download an entire model immediately before participating, because every model weight could get updated in every round.
In contrast, local top-$k$ and \fedssgd{} only update a limited number of parameters per round, so non-participating clients can stay relatively up to date with the current model, reducing the number of new parameters that must be downloaded immediately before participating.
This makes upload compression more important than download compression for local top-$k$ and \fedssgd{}.
Download compression is also less important for all three methods since residential Internet connections tend to reach far higher download than upload speeds \citep{goga2012speed}.
We include results here of overall compression (including upload and download), but break up the plots into separate upload and download components in the Appendix, Figure \ref{fig:cifar_complete}.

In all our experiments, we tune standard hyperparameters on the uncompressed runs, and we maintain these same hyperparameters for all compression schemes.
Details on which hyperparameters were chosen for each task can be found in Appendix \ref{appendix:experimentaldetails}.
\fedavg{} achieves compression by reducing the number of iterations carried out, so for these runs, we simply scale the learning rate schedule in the iteration dimension to match the total number of iterations that \fedavg{} will carry out.
We report results for each compression method over a range of hyperparameters: for local top-$k$, we adjust $k$; and for \fedssgd{} we adjust $k$ and the number of columns in the sketch (which controls the compression rate of the sketch).
We tune the number of local epochs and federated averaging batch size for \fedavg{}, but do not tune the learning rate decay for \fedavg{} because we find that \fedavg{} does not approach the baseline accuracy on our main tasks for even a small number of local epochs, where the learning rate decay has very little effect.

In the non-federated setting, momentum is typically crucial for achieving high performance, but in federating learning, momentum can be difficult to incorporate.
Each client could carry out momentum on its local gradients, but this is ineffective when clients participate only once or a few times.
Instead, the central aggregator can carry out momentum on the aggregated model updates.
For \fedavg{} and local top-$k$, we experiment with ($\rho_g=0.9$) and without ($\rho_g=0$) this global momentum.
For each method, neither choice of $\rho_g$ consistently performs better across our tasks, reflecting the difficulty of incorporating momentum.
In contrast, \fedssgd{} incorporates momentum seamlessly due to the linearity of our compression operator (see Section \ref{subsec:alg}); we use a momentum parameter of 0.9 in all experiments.

In all plots of performance vs. compression, each point represents a trained model, and for clarity, we plot only the Pareto frontier over hyperparameters for each method.
Figures \ref{fig:cifar_all_complete} and \ref{fig:femnist_gpt_all_complete} in the Appendix show results for all runs that converged.

\vspace*{-8pt}
\subsection{CIFAR (ResNet9)}
\vspace*{-2pt}
CIFAR10 and CIFAR100 \citep{krizhevsky2009learning} are image classification datasets with 60,000 $32\times32$px color images distributed evenly over 10 and 100 classes respectively (50,000/10,000 train/test split).
They are benchmark computer vision datasets, and although they lack a natural non-i.i.d.\ partitioning, we artificially create one by giving each client images from only a single class.
For CIFAR10 (CIFAR100) we use 10,000 (50,000) clients, yielding 5 (1) images per client.
Our 7M-parameter model architecture, data preprocessing, and most hyperparameters follow \citet{davidpage}, with details in Appendix \ref{appendix:cifar}.
We report accuracy on the test datasets.

Figure \ref{fig:cifar} shows test accuracy vs. compression for CIFAR10 and CIFAR100.
\fedavg{} and local top-$k$ both struggle to achieve significantly better results than uncompressed SGD.
Although we ran a large hyperparameter sweep, many runs simply diverge, especially for higher compression (local top-$k$) or more local iterations (\fedavg{}).
We expect this setting to be challenging for \fedavg{}, since running multiple gradient steps on only one or a few data points, especially points that are not representative of the overall distribution, is unlikely to be productive.
And although local top-$k$ can achieve high upload compression, download compression is reduced to almost $1\times$, since summing sparse gradients from many workers, each with very different data, leads to a nearly dense model update each round.

\vspace*{-8pt}
\subsection{FEMNIST (ResNet101)}
\vspace*{-2pt}
The experiments above show that \fedssgd{} significantly outperforms competing methods in the regime of very small local datasets and non-i.i.d. data.
In this section we introduce a task designed to be more favorable for \fedavg{}, and show that \fedssgd{} still performs competitively.

\begin{figure}[t]
\centering

\includegraphics[width=210px]{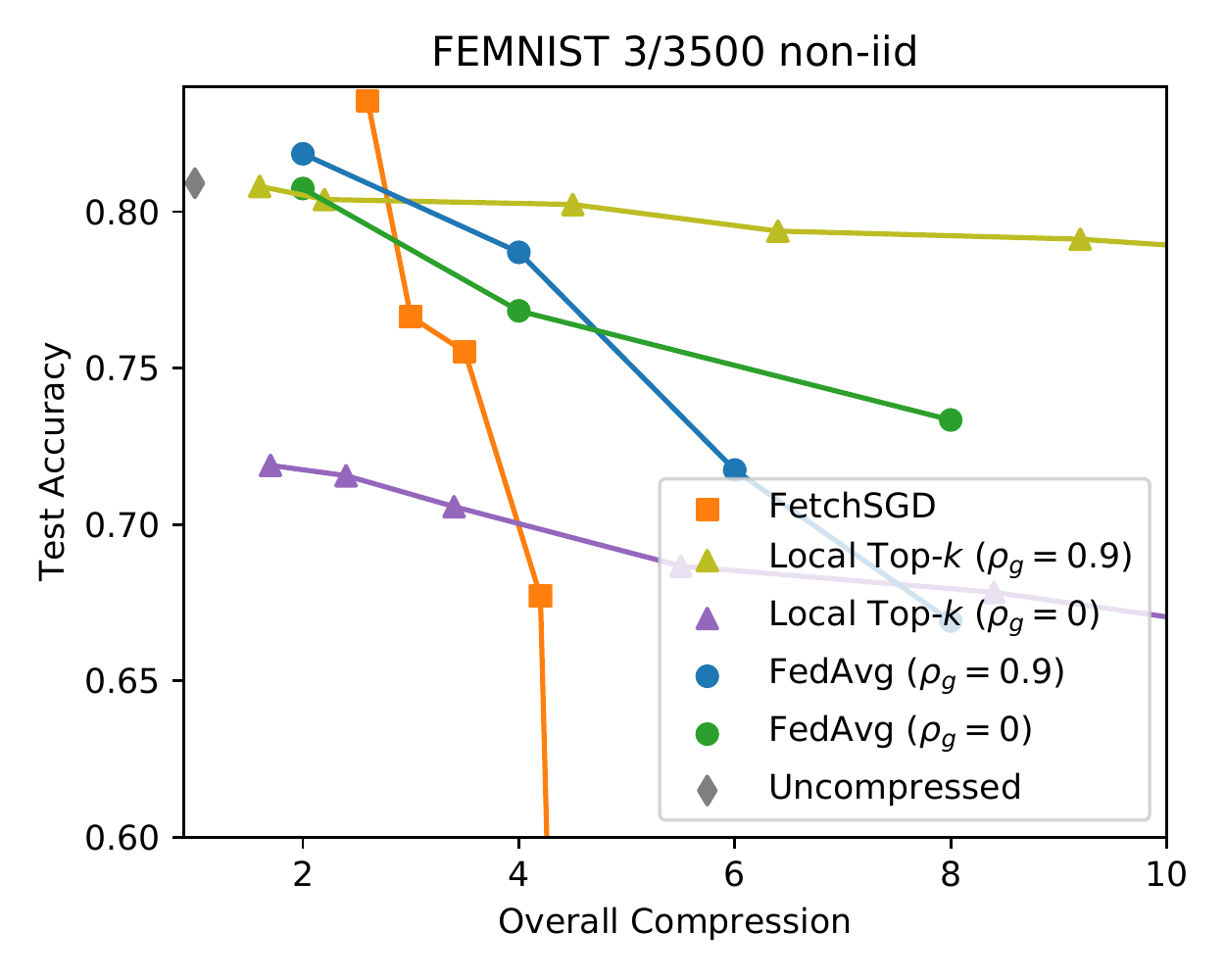}
  \vspace{-.5cm} %

\caption{Test accuracy on FEMNIST. The dataset is not very non-i.i.d., and has relatively large local datasets, but \fedssgd{} is still competitive with \fedavg{} and local top-$k$ for lower compression.
}
\vspace{-10pt}
\label{fig:femnist}

\end{figure}

\begin{figure*}[t]
\centering
\begin{subfigure}[b]{0.49\textwidth}
     \includegraphics[width=0.95\textwidth]{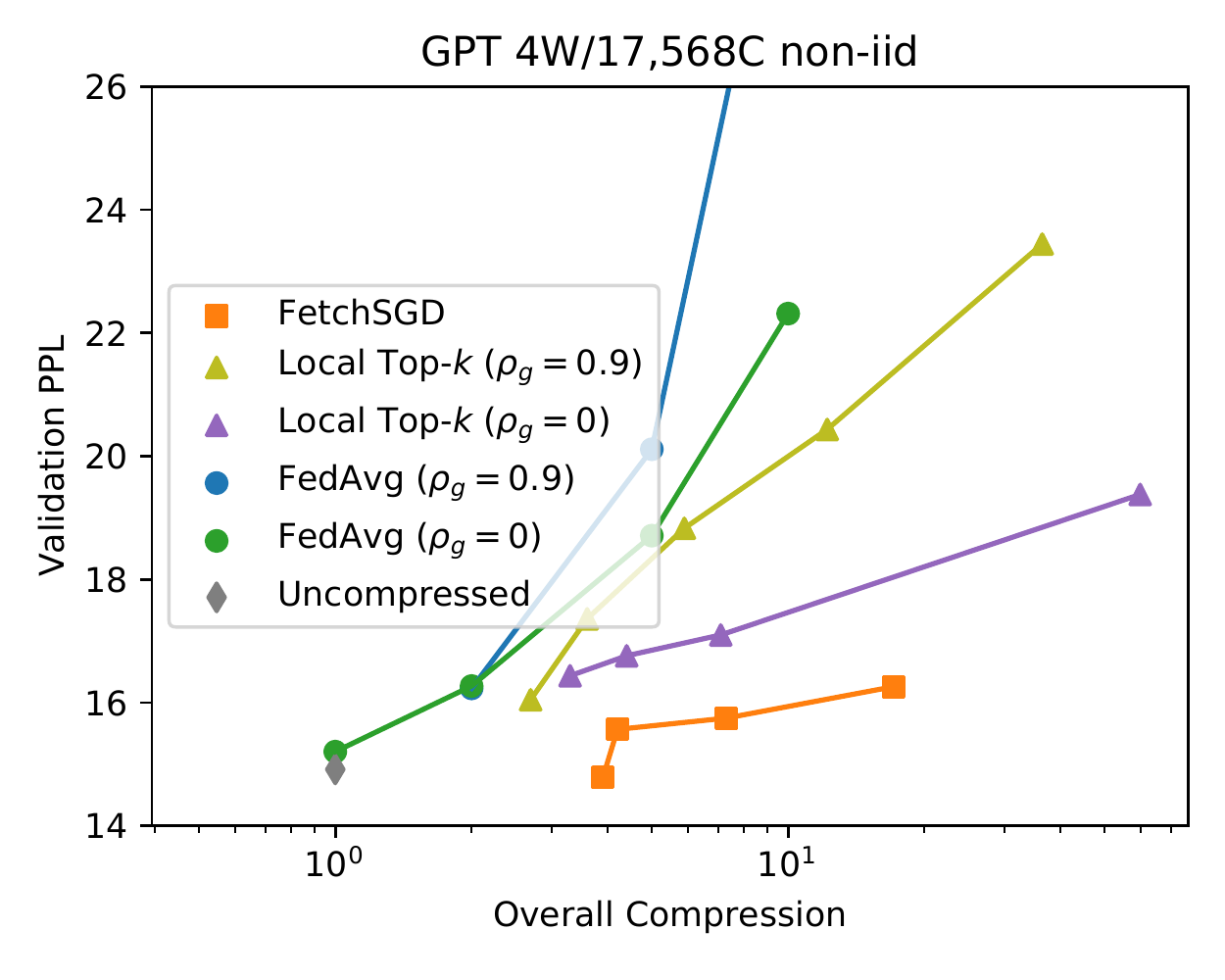}
\end{subfigure}
\begin{subfigure}[b]{0.49\textwidth}
    \includegraphics[width=0.95\textwidth]{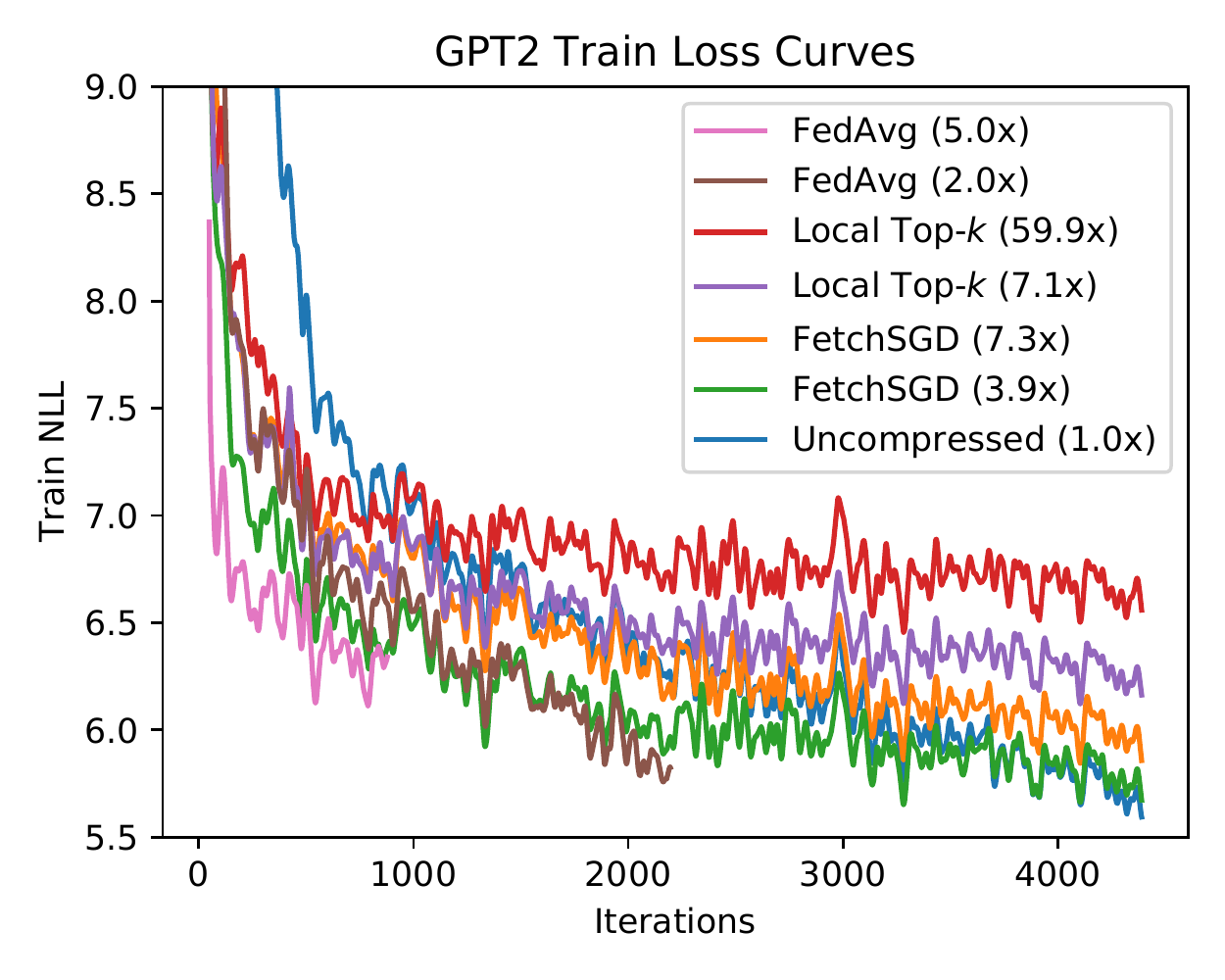}
\end{subfigure}
\vspace{-.5cm} %
\caption{Left: Validation perplexity achieved by finetuning GPT2-small on PersonaChat. \fedssgd{} achieves $3.9\times$ compression without loss in accuracy over uncompressed SGD, and it consistently achieves lower perplexity than \fedavg{} and top-$k$ runs with similar compression. Right: Training loss curves for representative runs. Global momentum hinders local top-$k$ in this case, so local top-$k$ runs with $\rho_g=0.9$ are omitted here to increase legibility.}
\label{fig:gpt2}
\vspace{-0.4cm}
\end{figure*}
Federated EMNIST is an image classification dataset with 62 classes (upper- and lower-case letters, plus digits) \citep{femnist}, which is formed by partitioning the EMNIST dataset \citep{emnist} such that each client in FEMNIST contains characters written by a single person.
Experimental details, including our 40M-parameter model architecture, can be found Appendix \ref{appendix:femnist}.
We report final accuracies on the validation dataset.
The baseline run trains for a single epoch (\textit{i.e.}, each client participates once).

FEMNIST was introduced as a benchmark dataset for \fedavg{}, and it has relatively large local datasets ($\sim200$ images per client).
The clients are split according to the person who wrote the character, yielding a data distribution closer to i.i.d.\ than our per-class splits of CIFAR10.
To maintain a reasonable overall batch size, only three clients participate each round, reducing the need for a linear compression operator.
Despite this, \fedssgd{} performs competitively with both \fedavg{} and local top-$k$ for some compression values, as shown in Figure \ref{fig:femnist}.

For low compression, \fedssgd{} actually outperforms the uncompressed baseline, likely because updating only $k$ parameters per round regularizes the model.
Interestingly, local top-$k$ using global momentum significantly outperforms other methods on this task, though we are not aware of prior work suggesting this method for federated learning.
Despite this surprising observation, local top-$k$ with global momentum suffers from divergence and low accuracy on our other tasks, and it lacks any theoretical guarantees.

\vspace*{-8pt}
\subsection{PersonaChat (GPT2)}
\vspace*{-2pt}
In this section we consider GPT2-small \citep{radford2019language}, a transformer model with 124M parameters that is used for language modeling.
We finetune a pretrained GPT2 on the PersonaChat dataset, a chit-chat dataset consisting of conversations between Amazon Mechanical Turk workers who were assigned faux personalities to act out \citep{persona}.
The dataset has a natural non-i.i.d.\ partitioning into 17,568 clients based on the personality that was assigned.
Our experimental procedure follows \citet{convai}. The baseline model trains for a single epoch, meaning that no local state is possible, and we report the final perplexity (a standard metric for language models; lower is better) on the validation dataset in Figure~\ref{fig:gpt2}.

Figure \ref{fig:gpt2} also plots loss curves (negative log likelihood) achieved during training for some representative runs.
Somewhat surprisingly, all the compression techniques outperform the uncompressed baseline early in training, but most saturate too early, when the error introduced by the compression starts to hinder training.

Sketching outperforms local top-$k$ for all but the highest levels of compression, because local top-$k$ relies on local state for error feedback, which is impossible in this setting.
We expect this setting to be challenging for \fedavg{}, since running multiple gradient steps on a single conversation which is not representative of the overall distribution is unlikely to be productive.

\vspace{-0.3cm}
\section{Discussion}
Federated learning has seen a great deal of research interest recently, particularly in the domain of communication efficiency.
A considerable amount of prior work focuses on decreasing the total number of communication rounds required to converge, without reducing the communication required in each round.
In this work, we complement this body of work by introducing \fedssgd{}, an algorithm that reduces the amount of communication required each round, while still conforming to the other constraints of the federated setting.
We particularly want to emphasize that \fedssgd{} easily addresses the setting of non-i.i.d. data, which often complicates other methods.
The optimal algorithm for many federated learning settings will no doubt combine efficiency in number of rounds and efficiency within each round, and we leave an investigation into optimal ways of combining these approaches to future work.

\newpage
\section*{Acknowledgements}
This research was supported in part by NSF BIGDATA awards IIS-1546482, IIS-1838139, NSF CAREER award IIS-1943251, NSF CAREER grant 1652257, NSF GRFP grant DGE 1752814, ONR Award N00014-18-1-2364 and the Lifelong Learning Machines program from DARPA/MTO.
RA would like to acknowledge support provided by Institute for Advanced Study. 

In addition to NSF CISE Expeditions Award CCF-1730628, this research is supported by gifts from Alibaba, Amazon Web Services, Ant Financial, CapitalOne, Ericsson, Facebook, Futurewei, Google, Intel, Microsoft, Nvidia, Scotiabank, Splunk and VMware.

\bibliographystyle{plainnat}
\bibliography{main}
\appendix
\onecolumn
The Appendix is organized as follows:
\begin{itemize}
    \item Appendix \ref{appendix:experimentaldetails} lists hyperparameters and model architectures used in all experiments, and includes plots with additional experimental data, including results broken down into upload, download and overall compression.
    \item Appendix \ref{appendix:theory} gives full proofs of convergence for \fedssgd{}.
    \item Appendix \ref{appendix:countsketch} describes the Count Sketch data structure and how it is used in \fedssgd{}.
    \item Appendix \ref{appendix:sliding_window} provides the high level idea of the sliding window model and describes how to extend a sketch data structure to the sliding window setting.
\end{itemize}
\section{Experimental Details}
\label{appendix:experimentaldetails}
We run all experiments on commercially available NVIDIA Pascal, Volta and Turing architecture GPUs.
\subsection{CIFAR}
\label{appendix:cifar}
In all non-\fedavg{} experiments we train for 24 epochs, with 1$\%$ of clients participating each round, for a total of 2400 iterations.
We use standard train/test splits of 50000 training datapoints and 10000 validation.
We use a triangular learning rate schedule which peaks at epoch 5. 
We use the maximum peak learning rate for which the uncompressed runs converge: $0.3$ for CIFAR10, and $0.2$ for CIFAR100.
We use this learning rate schedule for all compressed runs.
\fedavg{} runs for fewer than 24 epochs, so we compress the learning rate schedule in the iteration dimension accordingly.
We do not tune the learning rate separately for any of the compressed runs.

We split the datasets into 10,000 (CIFAR10) and 50,000 (CIFAR100) clients, each of which has 5 (CIFAR10) and 1 (CIFAR100) data point(s) from a single target class.
In each round, 1\% of clients participate, leading to a total batch size of 500 for both datasets (100 clients with 5 data points for CIFAR10, and 500 clients with 1 data point for CIFAR100).
We augment the data during training with random crops and random horizontal flips, and we normalize the images by the dataset mean and standard deviation during training and testing.
We use a modified ResNet9 architecture with 6.5M parameters for CIFAR10, and 6.6M parameters for CIFAR100.
We do not use batch normalization in any experiments, since it is ineffective with the very small local batch sizes we use.
Most of these training procedures, and the modified ResNet9 architecture we use, are drawn from the work of \citet{davidpage}.

\fedssgd{}, \fedavg{} and local top-$k$ each have unique hyperparameters that we search over.
For \fedssgd{}, we try a grid of values for $k$ and the number of columns in the sketch.
For $k$ we try values of [10, 25, 50, 75, 100]~$\times 10^3$.
For the number of columns we try values of [325, 650, 1300, 2000, 3000]~$\times 10^3$.
We also tune $k$ for local top-$k$, trying values of [325, 650, 1300, 2000, 3000, 5000]~$\times 10^3$.
We present results for local top-$k$ with and without global momentum, but not with local momentum:
with such a low participation rate, we observe anecdotally that local momentum performs poorly, since the momentum is always stale, and maintaining local momentum and error accumulation vectors for the large number of clients we experiment with is computationally expensive.
The two hyperparameters of interest in \fedavg{} are the total number of global epochs to run (which determines the compression), and the number of local epochs to perform.
We run a grid search over global epochs of [6, 8, 12] (corresponding to $4\times$, $3\times$, and $2\times$ compression), and local epochs of [2,3,5].

Figure \ref{fig:cifar_complete} shows the Pareto frontier of results with each method for CIFAR10 and CIFAR100 broken down into upload, download, and overall compression.
Figure \ref{fig:cifar_all_complete} shows all runs that converged for the two datasets.
For CIFAR10, 1 \fedssgd{} run, 3 local top-$k$ runs, and all \fedavg{} runs using global momentum diverged.
For CIFAR100, 1 local top-$k$ run and all \fedavg{} runs using global momentum diverged.

\begin{figure*}[th]
\centering
\begin{subfigure}[b]{0.49\textwidth}
    \includegraphics[width=\textwidth]{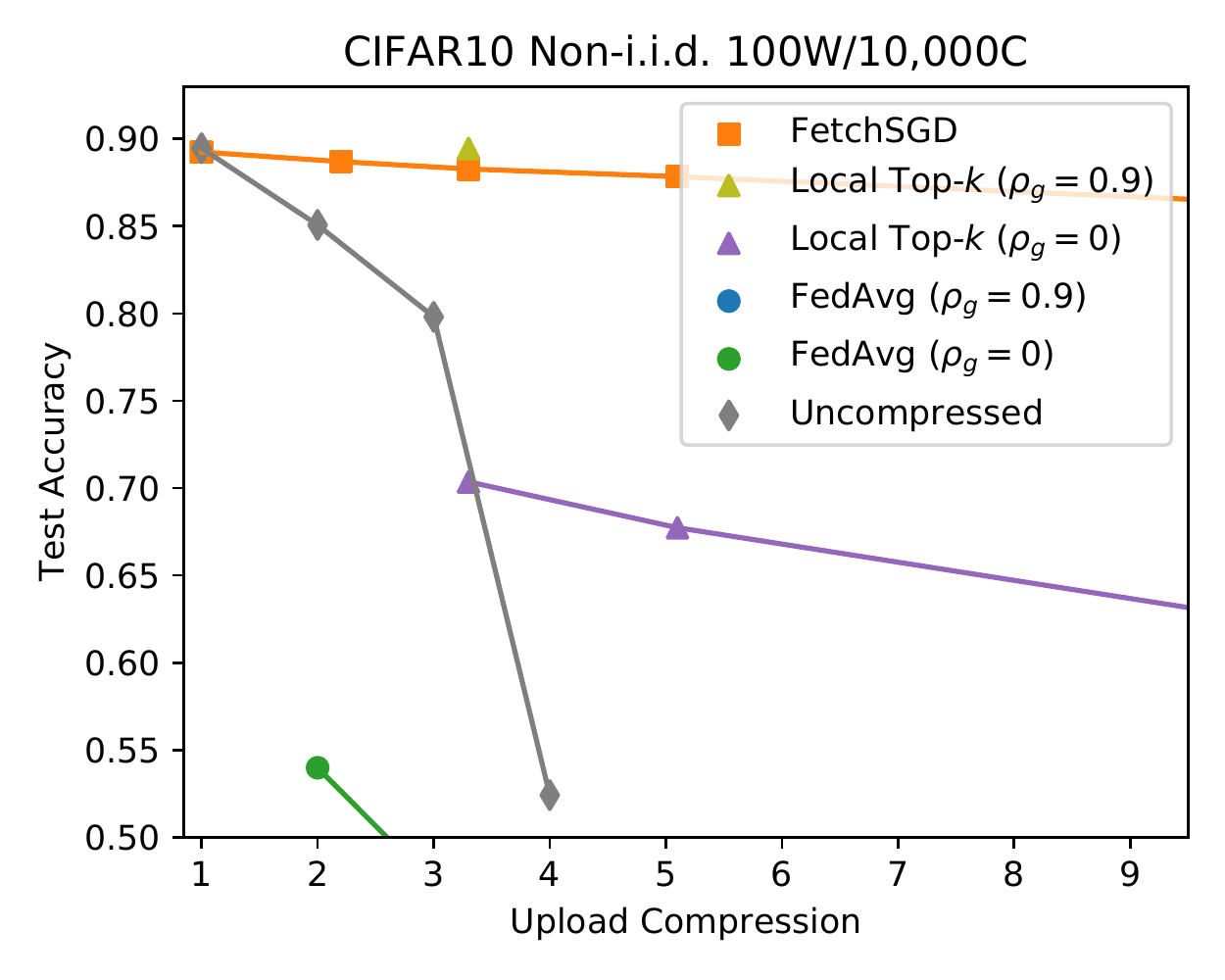}
    \caption{CIFAR10 Upload Compression}
\end{subfigure}
\begin{subfigure}[b]{0.49\textwidth}
    \includegraphics[width=\textwidth]{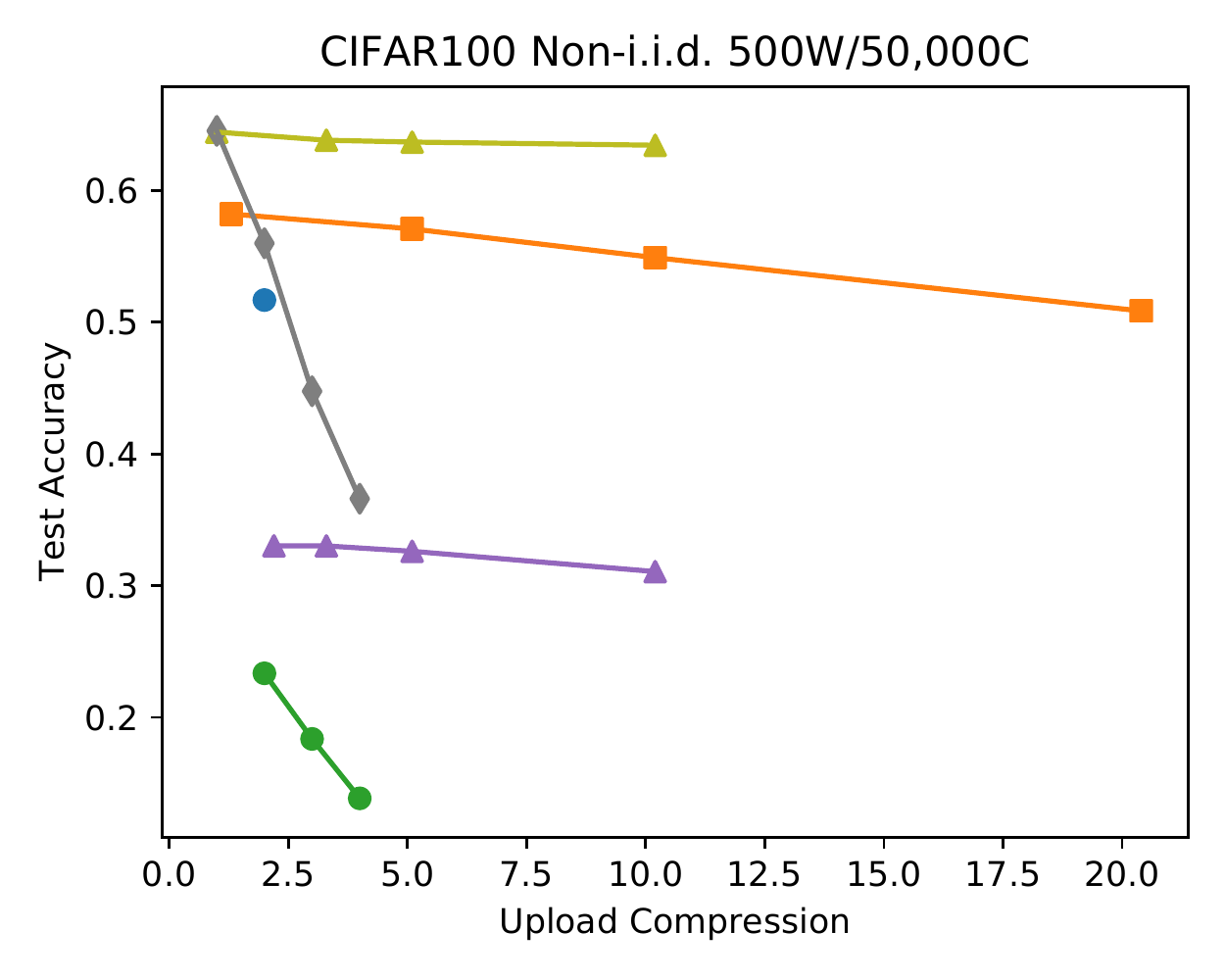}
    \caption{CIFAR100 Upload Compression}
\end{subfigure}
\\
\begin{subfigure}[b]{0.49\textwidth}
    \includegraphics[width=\textwidth]{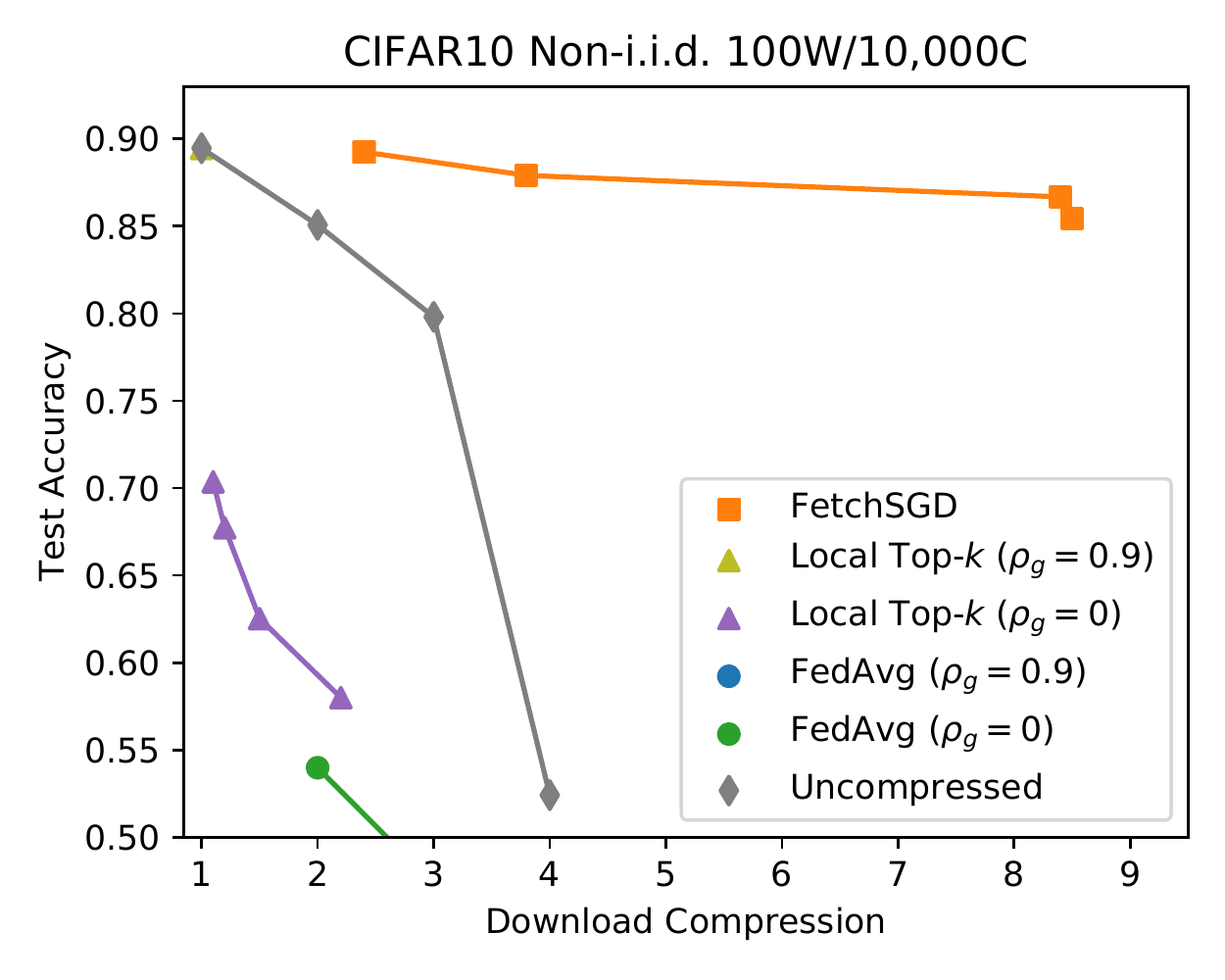}
    \caption{CIFAR10 Download Compression}
\end{subfigure}
\begin{subfigure}[b]{0.49\textwidth}
    \includegraphics[width=\textwidth]{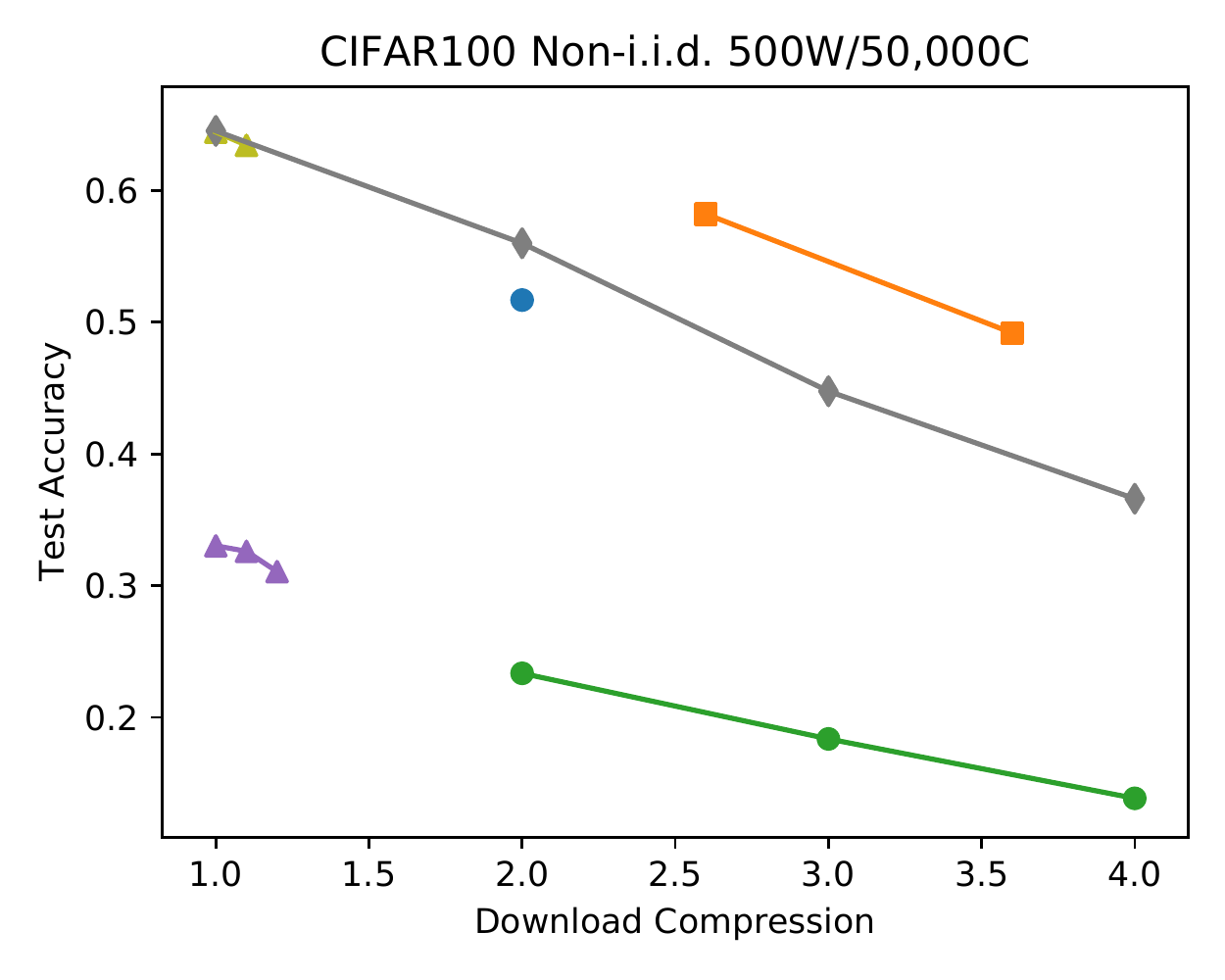}
    \caption{CIFAR100 Download Compression}
\end{subfigure}
\\
\begin{subfigure}[b]{0.49\textwidth}
    \includegraphics[width=\textwidth]{plots/cifar10_100w_10000c_noniid_updown.pdf}
    \caption{CIFAR10 Overall Compression}
\end{subfigure}
\begin{subfigure}[b]{0.49\textwidth}
    \includegraphics[width=\textwidth]{plots/cifar100_500w_50000c_noniid_updown.pdf}
    \caption{CIFAR100 Overall Compression}
\end{subfigure}
\caption{Upload (top), download (middle), and overall (bottom) compression for CIFAR10 (left) and CIFAR100 (right). To increase readability, each plot shows only the Pareto frontier of runs for the compression type shown in that plot. All runs that converged are shown in Figure \ref{fig:cifar_all_complete}.}
\label{fig:cifar_complete}
\end{figure*}

\begin{figure*}[th]
\centering
\begin{subfigure}[b]{0.49\textwidth}
    \includegraphics[width=\textwidth]{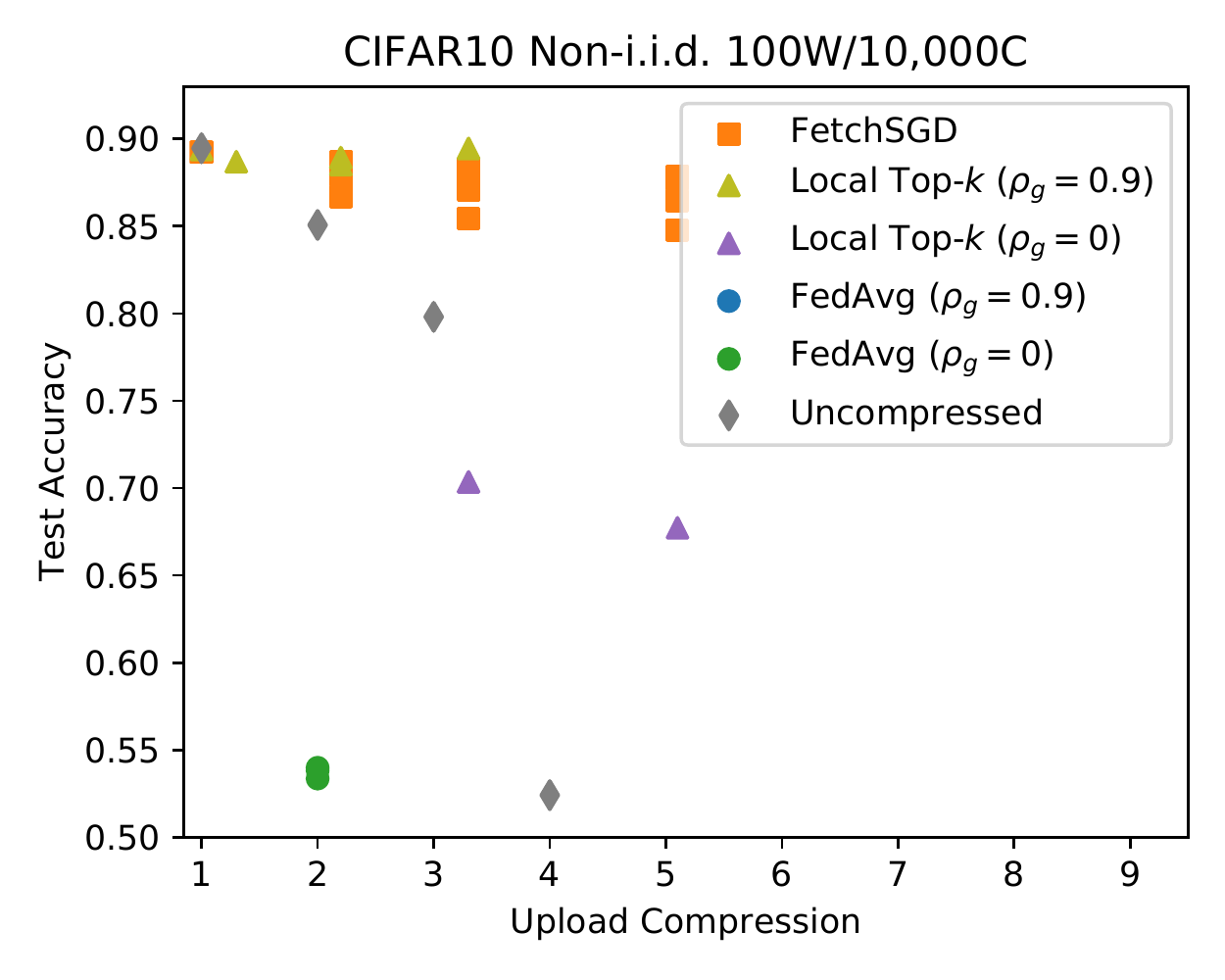}
    \caption{CIFAR10 Upload Compression}
\end{subfigure}
\begin{subfigure}[b]{0.49\textwidth}
    \includegraphics[width=\textwidth]{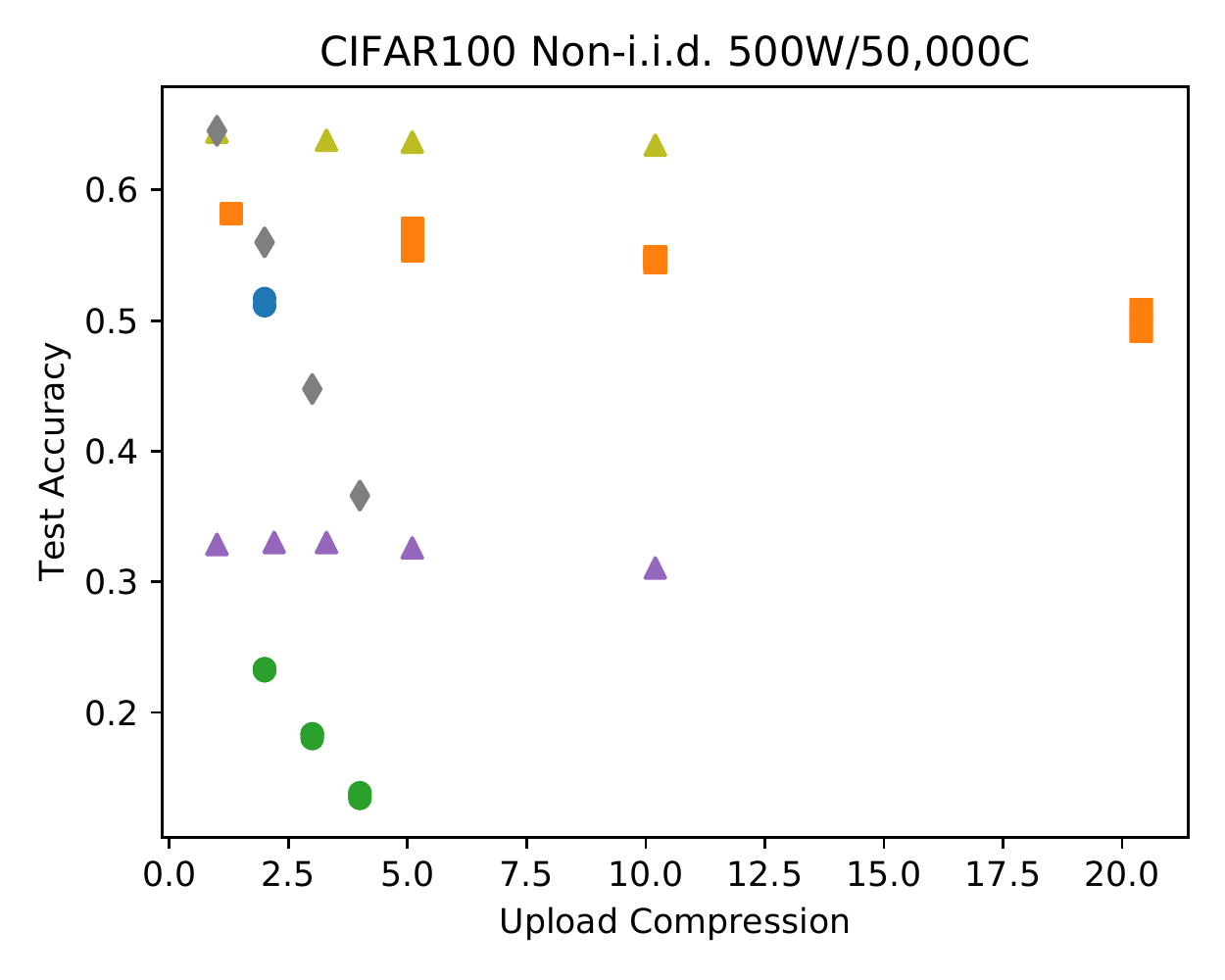}
    \caption{CIFAR100 Upload Compression}
\end{subfigure}
\\
\begin{subfigure}[b]{0.49\textwidth}
    \includegraphics[width=\textwidth]{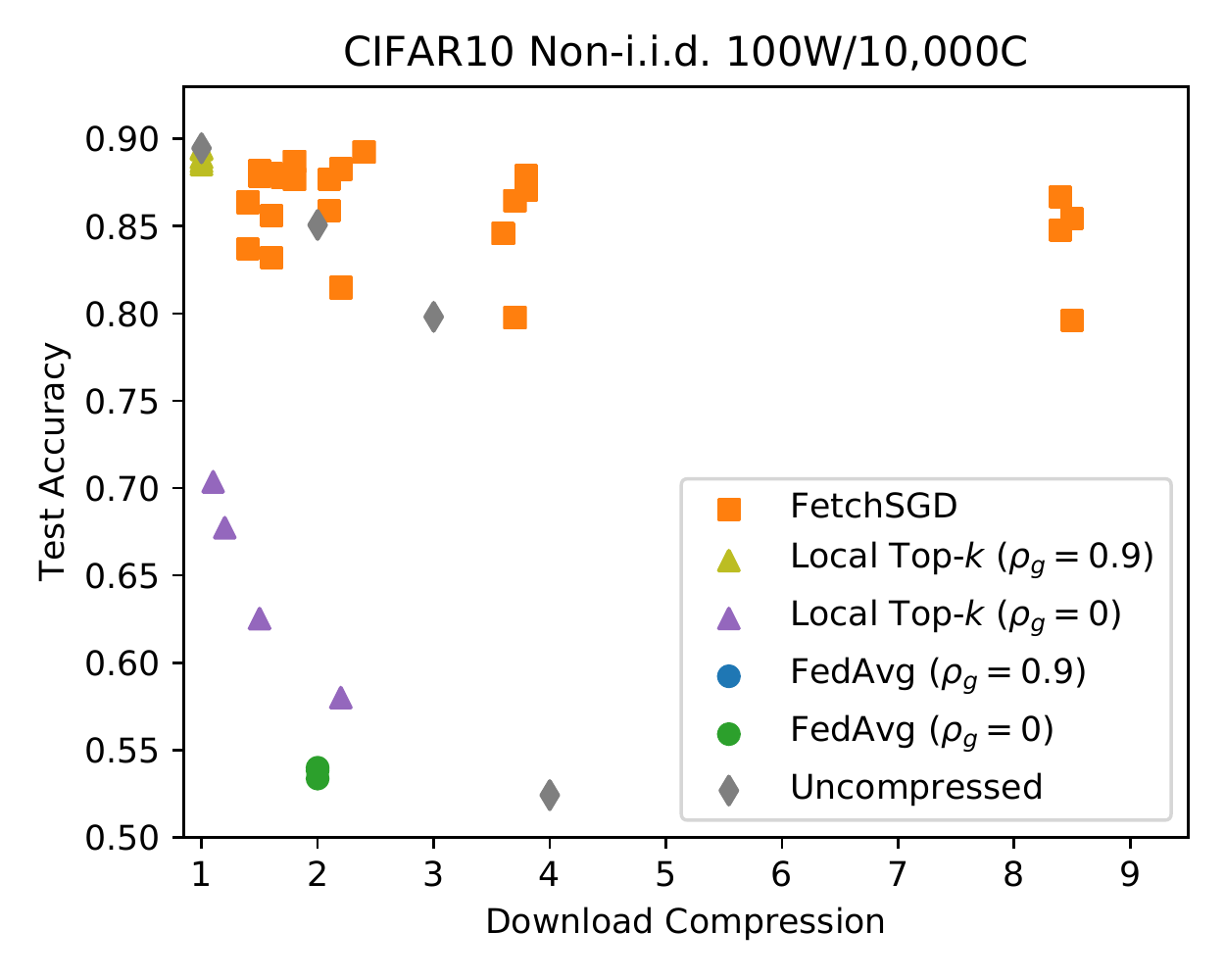}
    \caption{CIFAR10 Download Compression}
\end{subfigure}
\begin{subfigure}[b]{0.49\textwidth}
    \includegraphics[width=\textwidth]{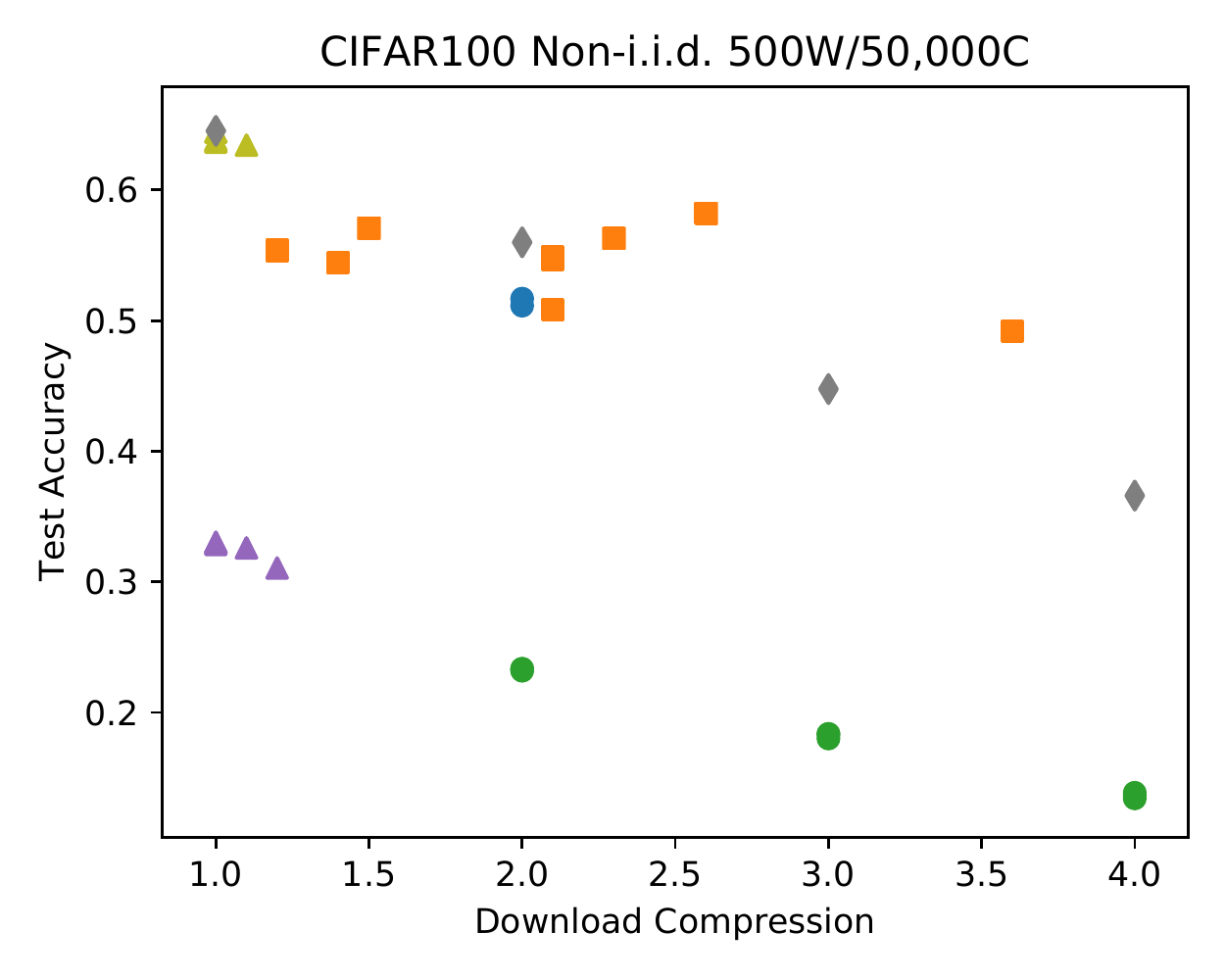}
    \caption{CIFAR100 Download Compression}
\end{subfigure}
\\
\begin{subfigure}[b]{0.49\textwidth}
    \includegraphics[width=\textwidth]{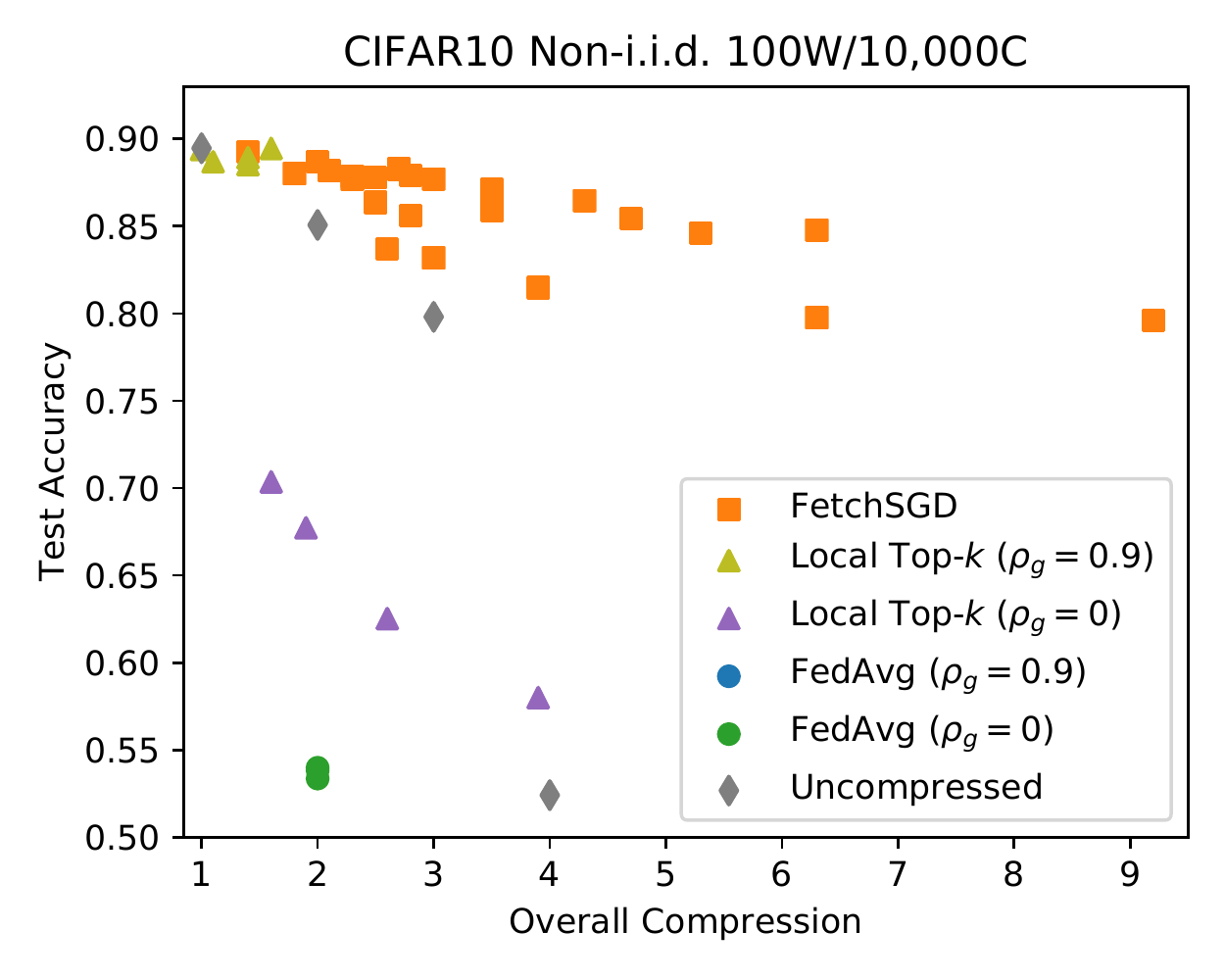}
    \caption{CIFAR10 Overall Compression}
\end{subfigure}
\begin{subfigure}[b]{0.49\textwidth}
    \includegraphics[width=\textwidth]{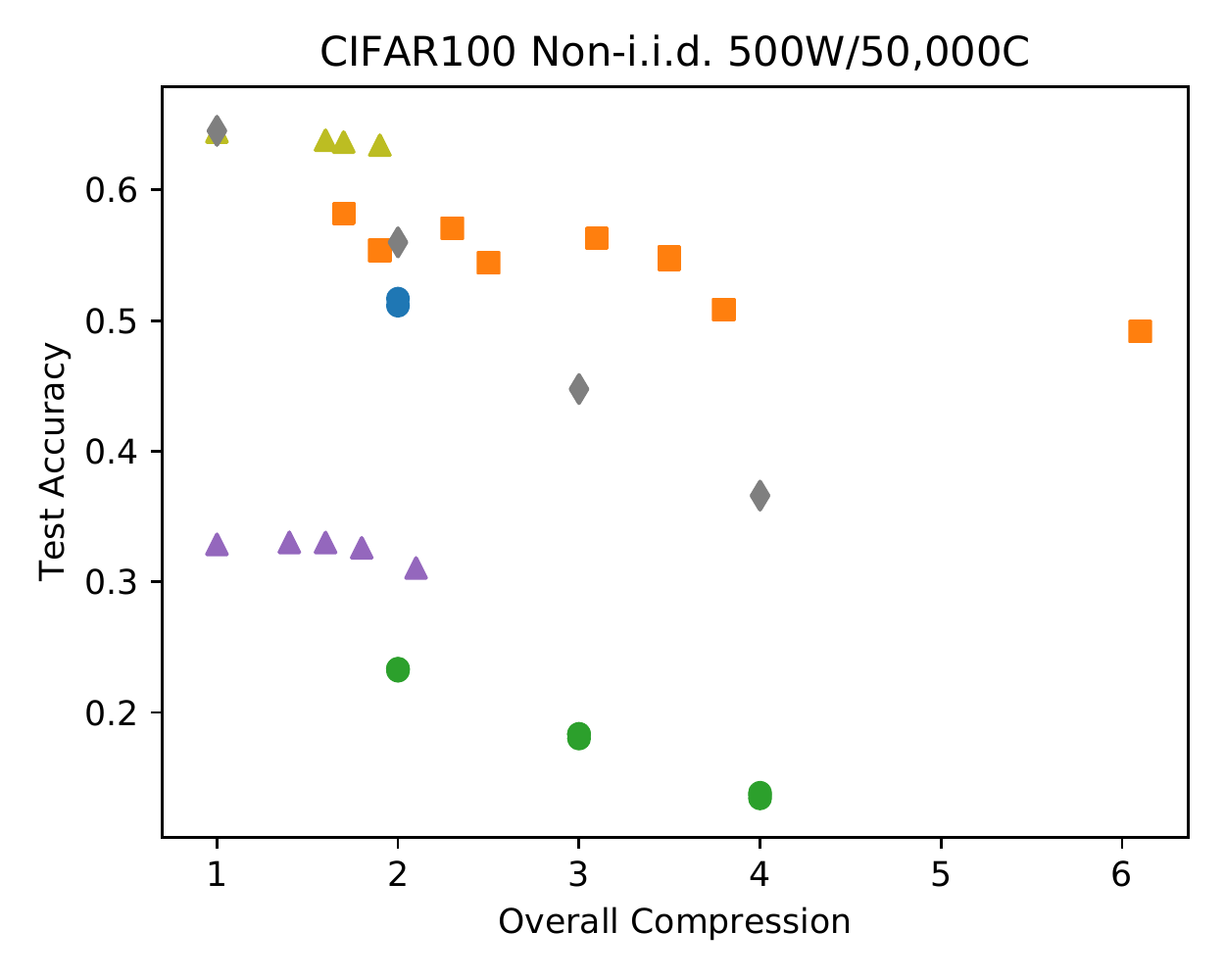}
    \caption{CIFAR100 Overall Compression}
\end{subfigure}
\caption{Upload (top), download (middle), and overall (bottom) compression for CIFAR10 (left) and CIFAR100 (right).}
\label{fig:cifar_all_complete}
\end{figure*}

\begin{figure*}[t]
\centering
\begin{subfigure}[b]{0.49\textwidth}
    \includegraphics[width=\textwidth]{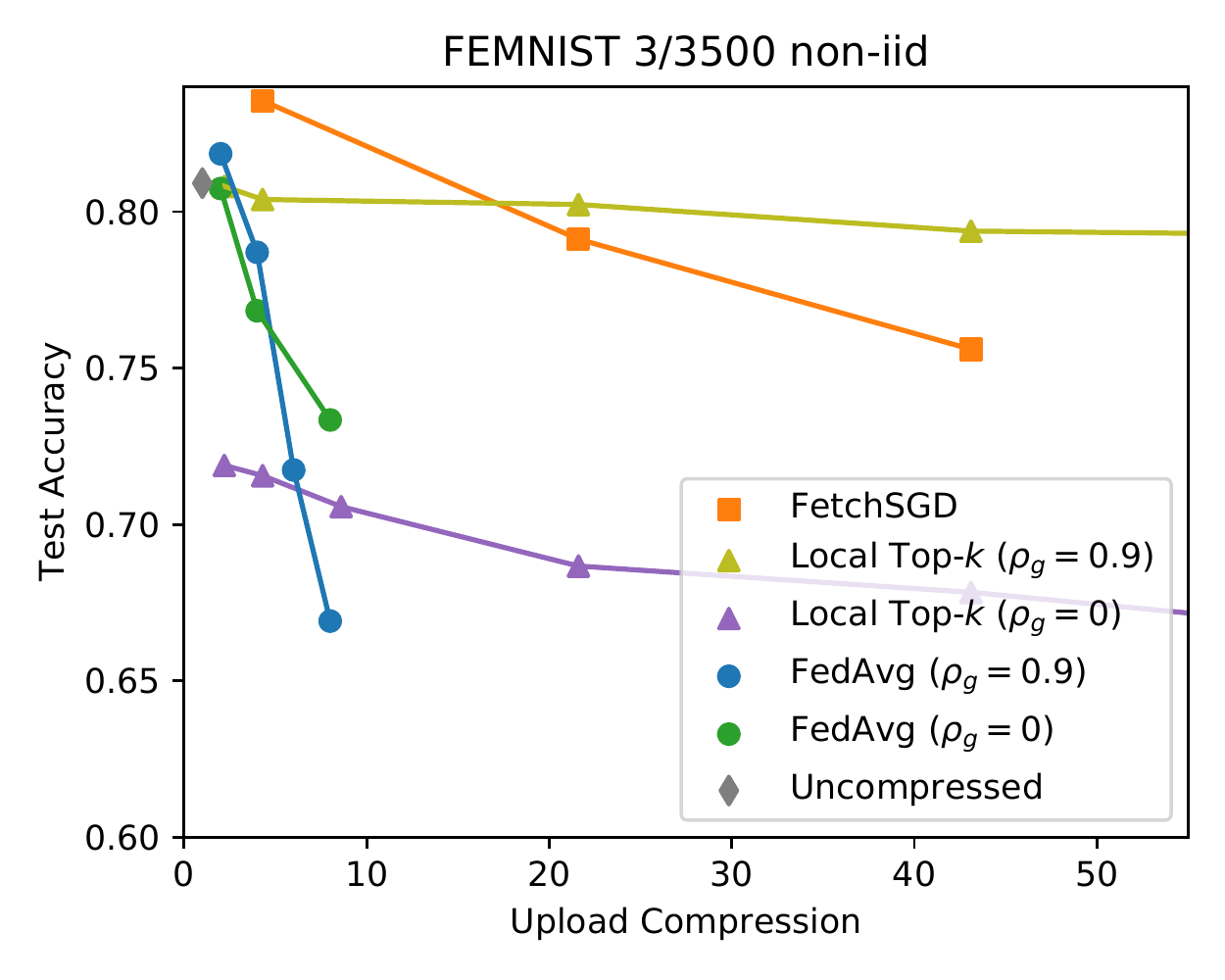}
    \caption{FEMNIST Upload Compression}
\end{subfigure}
\begin{subfigure}[b]{0.49\textwidth}
    \includegraphics[width=\textwidth]{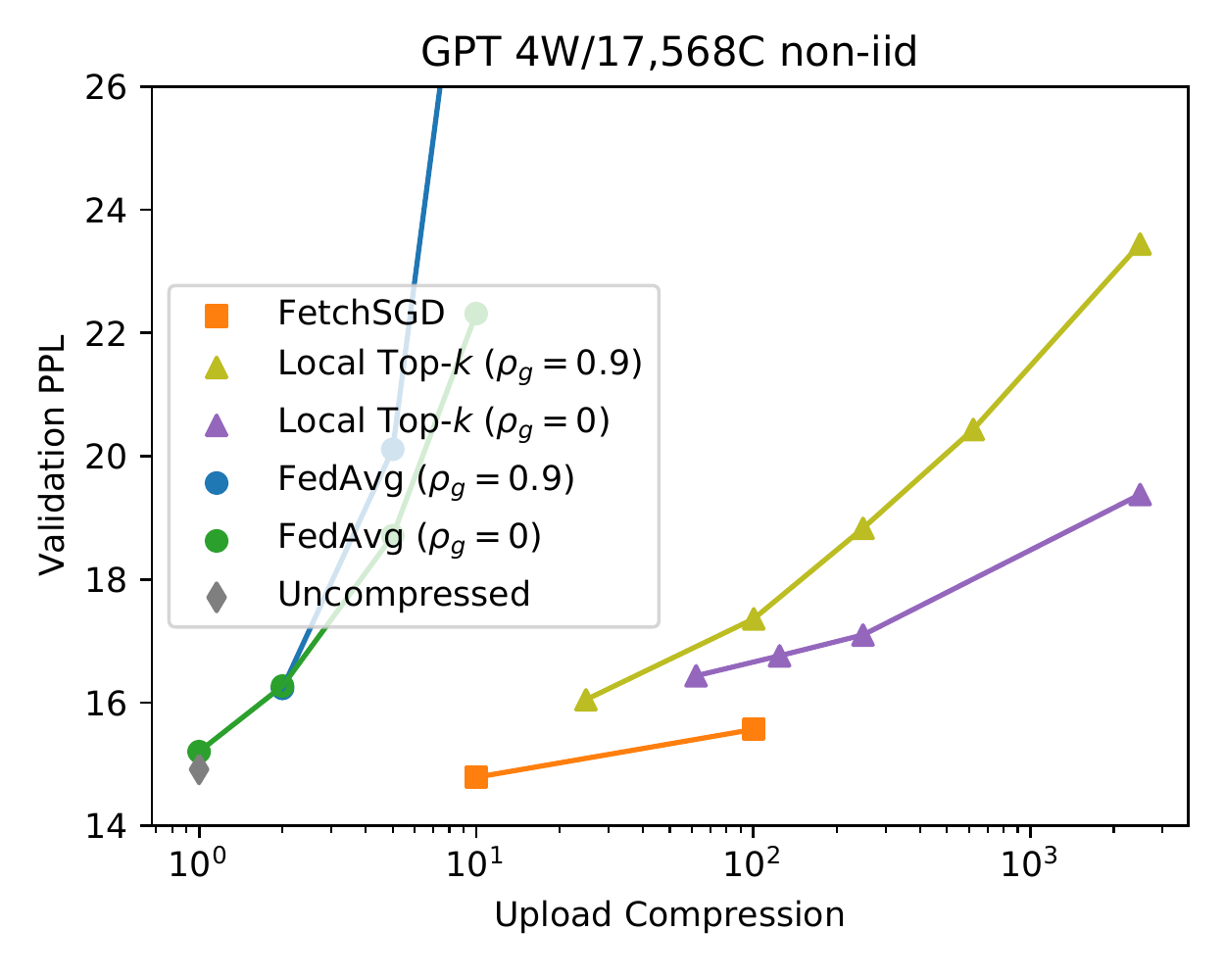}
    \caption{PersonaChat Upload Compression}
\end{subfigure}
\\
\begin{subfigure}[b]{0.49\textwidth}
    \includegraphics[width=\textwidth]{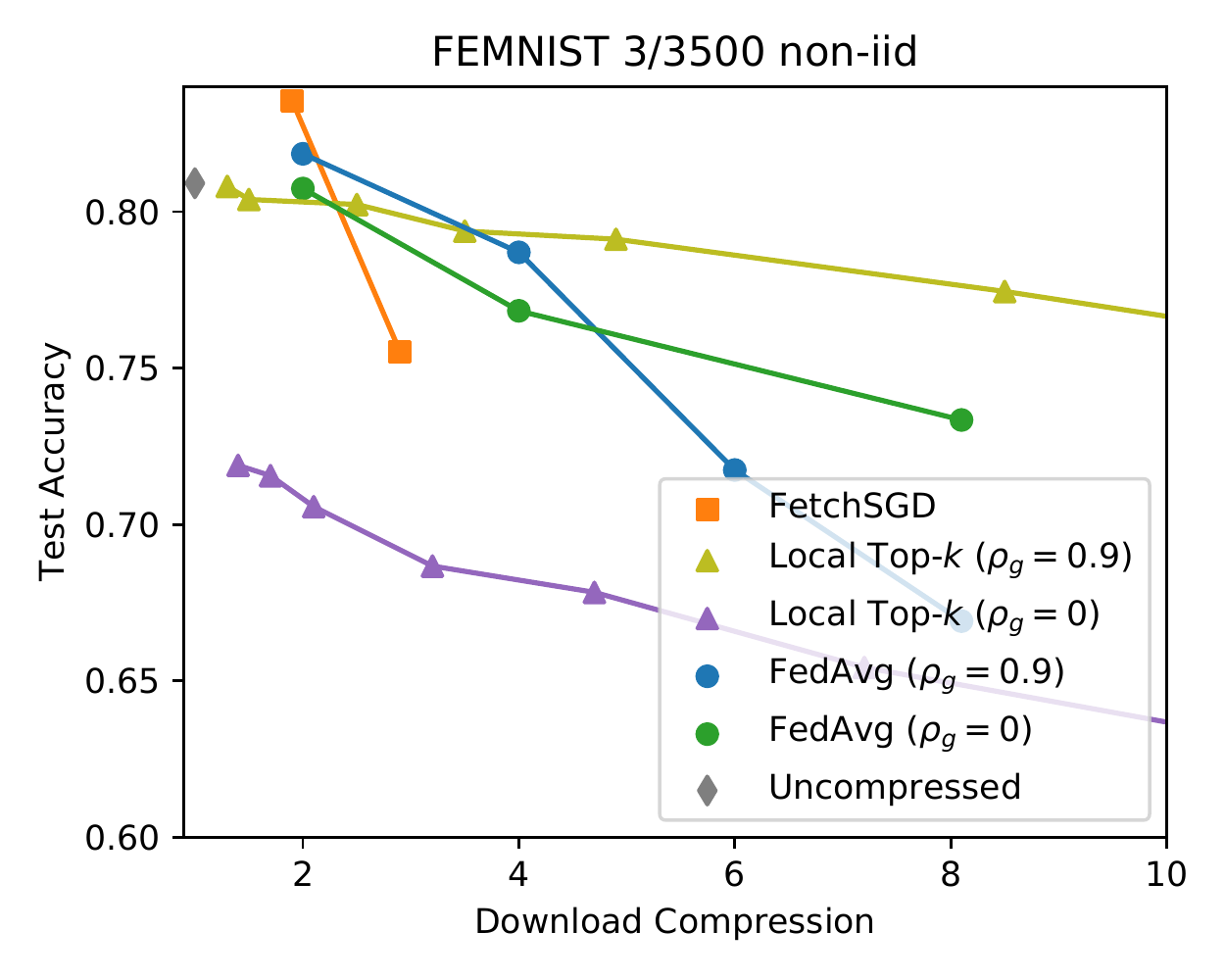}
    \caption{FEMNIST Download Compression}
\end{subfigure}
\begin{subfigure}[b]{0.49\textwidth}
    \includegraphics[width=\textwidth]{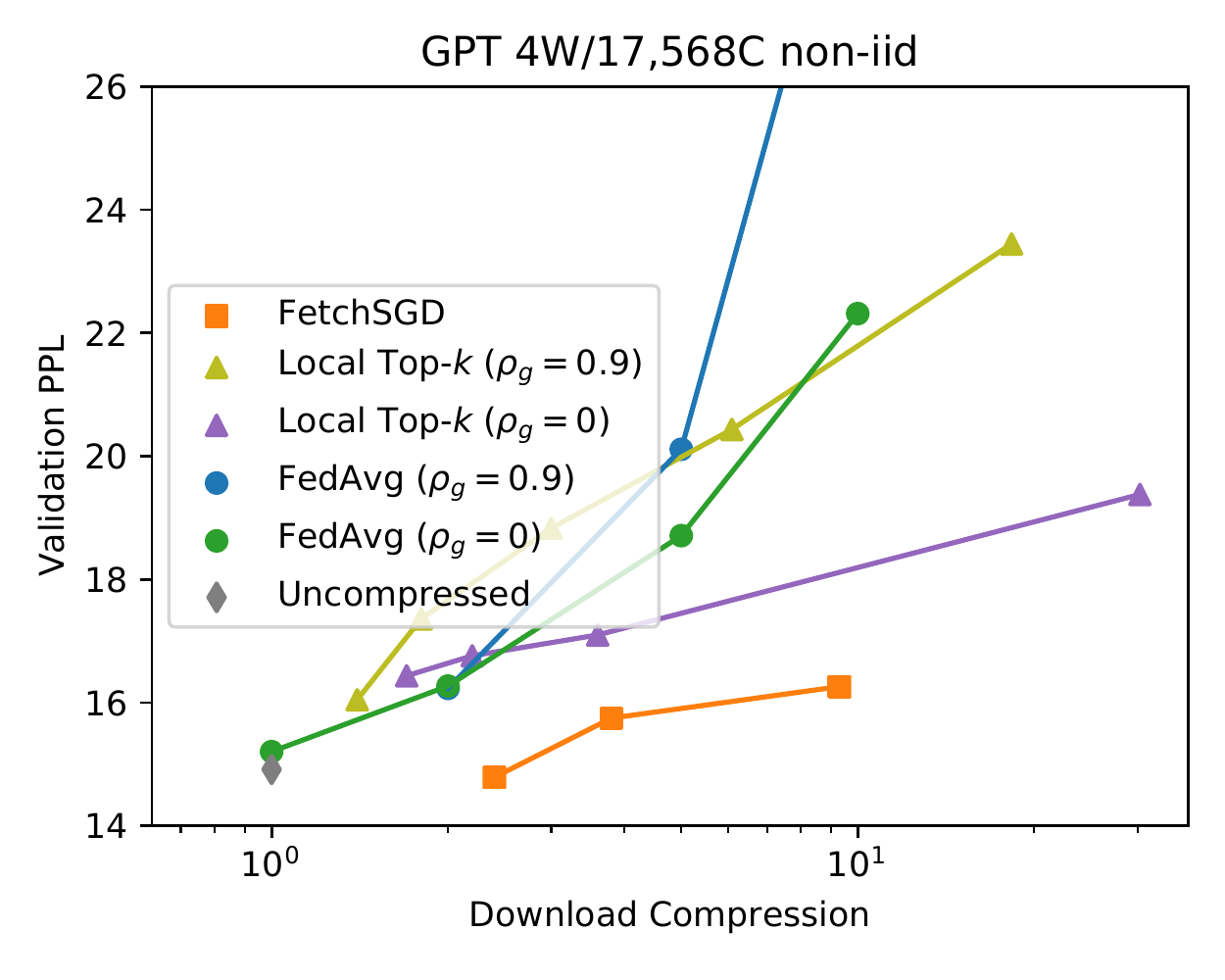}
    \caption{PersonaChat Download Compression}
\end{subfigure}
\\
\begin{subfigure}[b]{0.49\textwidth}
    \includegraphics[width=\textwidth]{plots/femnist_3w_3500c_noniid_updown.pdf}
    \caption{FEMNIST Overall Compression}
\end{subfigure}
\begin{subfigure}[b]{0.49\textwidth}
    \includegraphics[width=\textwidth]{plots/gpt2_scatter_updown.pdf}
    \caption{PersonaChat Overall Compression}
\end{subfigure}
\caption{Upload (top), download (middle), and overall (bottom) compression for FEMNIST (left) and PersonaChat (right). To increase readability, each plot shows only the Pareto frontier of runs for the compression type shown in that plot. All results are shown in Figure \ref{fig:femnist_gpt_all_complete}.}
\label{fig:femnist_gpt_complete}
\end{figure*}

\begin{figure*}[t]
\centering
\begin{subfigure}[b]{0.49\textwidth}
    \includegraphics[width=\textwidth]{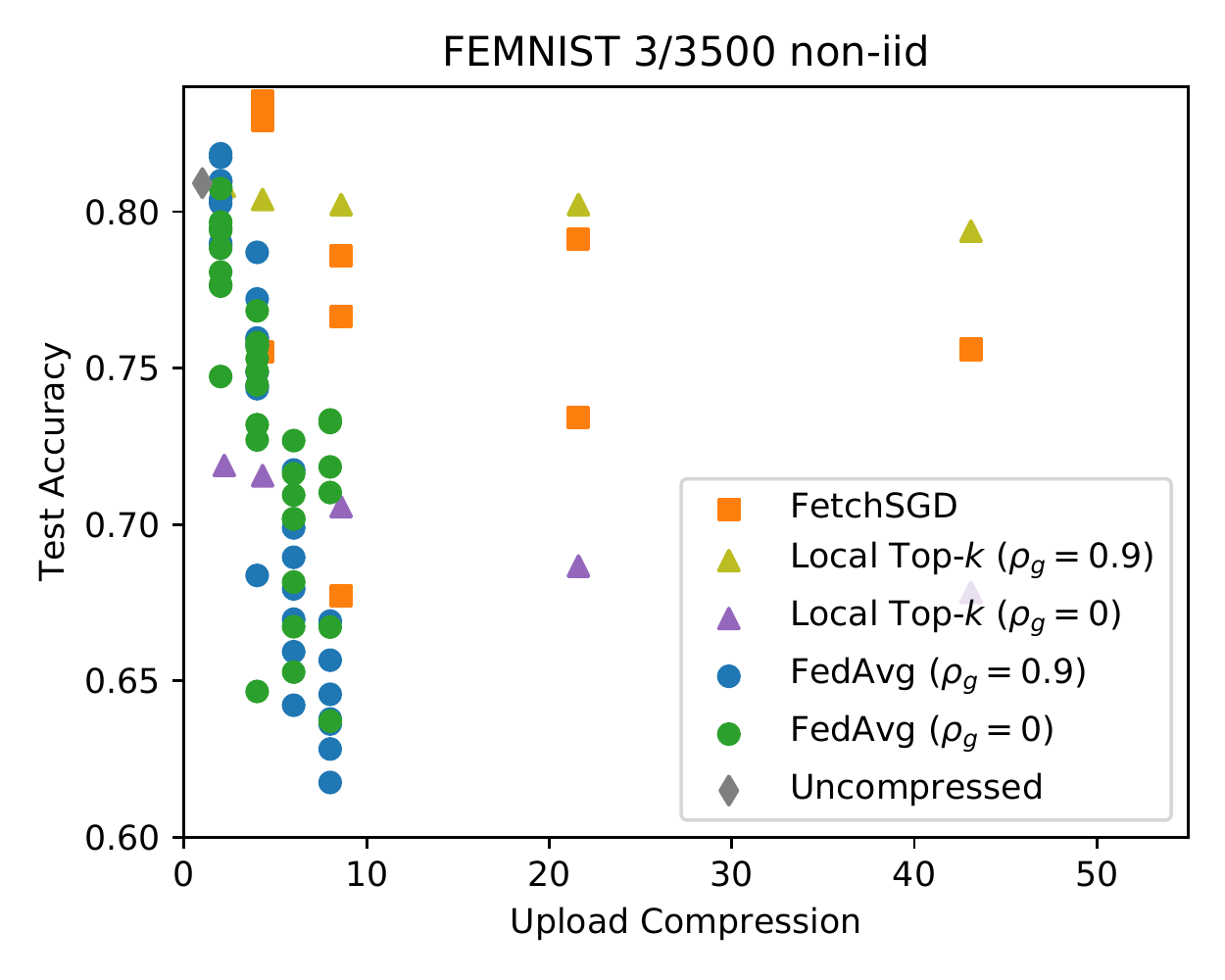}
    \caption{FEMNIST Upload Compression}
\end{subfigure}
\begin{subfigure}[b]{0.49\textwidth}
    \includegraphics[width=\textwidth]{plots/gpt2_scatter_upload.pdf}
    \caption{PersonaChat Upload Compression}
\end{subfigure}
\\
\begin{subfigure}[b]{0.49\textwidth}
    \includegraphics[width=\textwidth]{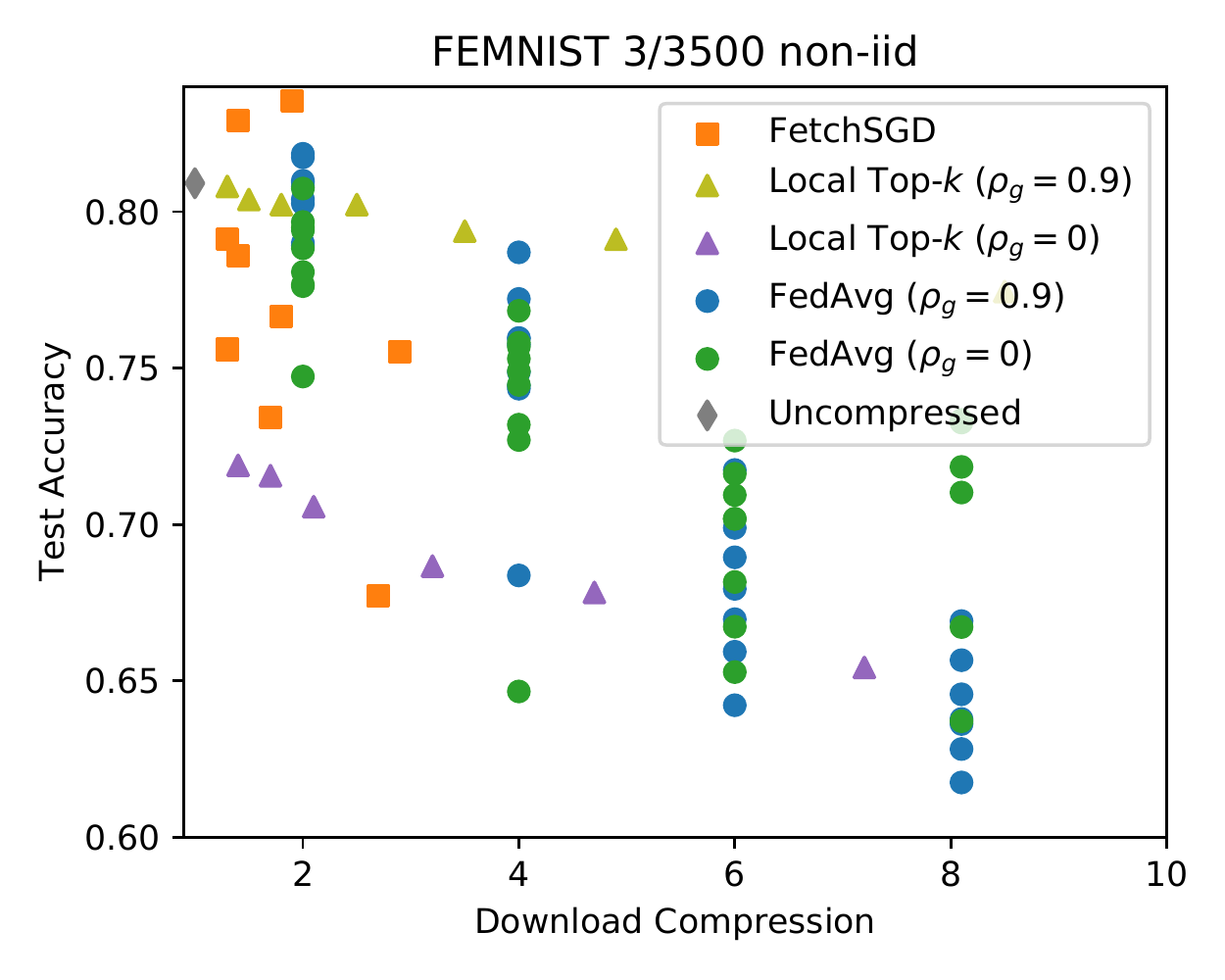}
    \caption{FEMNIST Download Compression}
\end{subfigure}
\begin{subfigure}[b]{0.49\textwidth}
    \includegraphics[width=\textwidth]{plots/gpt2_scatter_download.pdf}
    \caption{PersonaChat Download Compression}
\end{subfigure}
\\
\begin{subfigure}[b]{0.49\textwidth}
    \includegraphics[width=\textwidth]{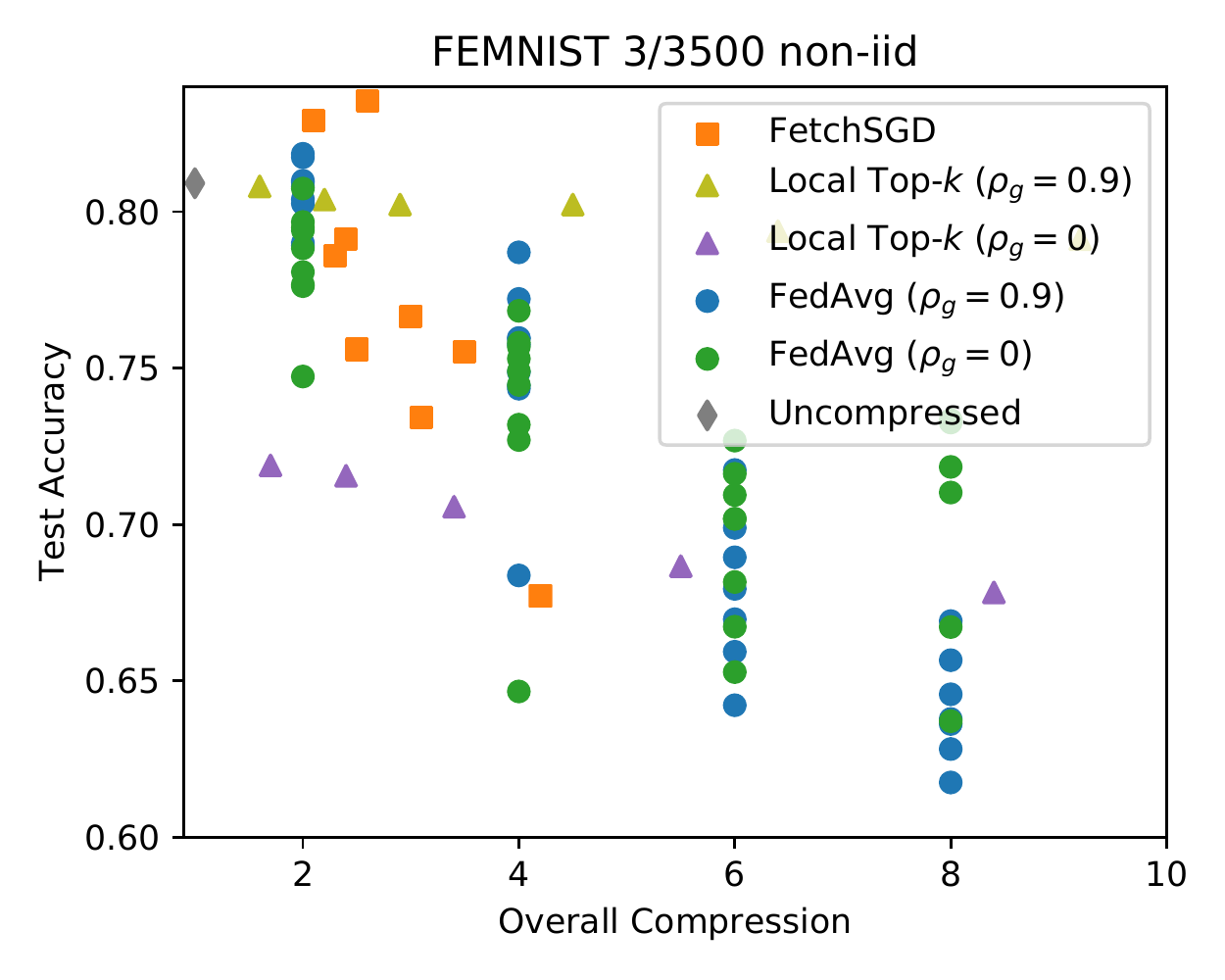}
    \caption{FEMNIST Overall Compression}
\end{subfigure}
\begin{subfigure}[b]{0.49\textwidth}
    \includegraphics[width=\textwidth]{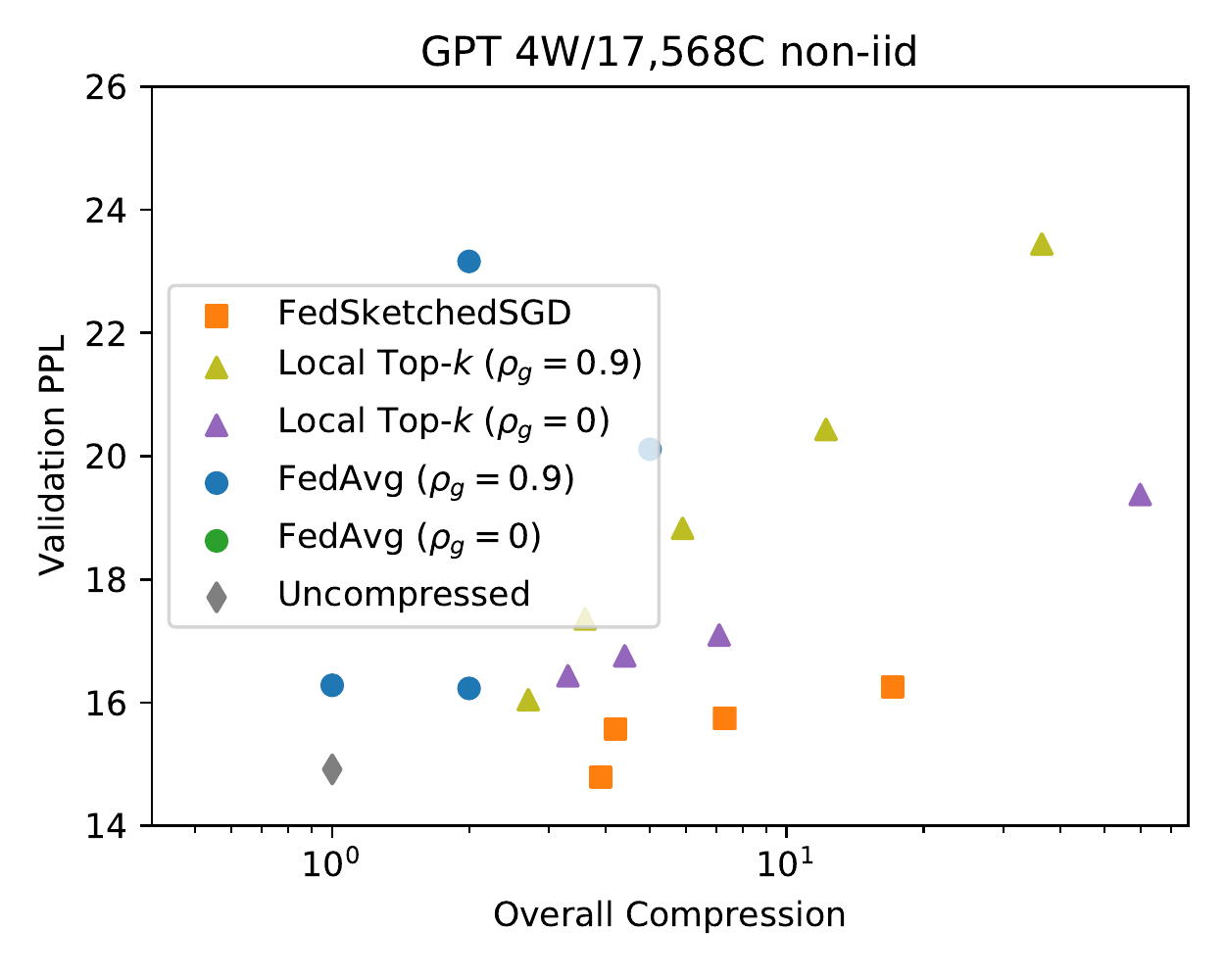}
    \caption{PersonaChat Overall Compression}
\end{subfigure}
\caption{Upload (top), download (middle), and overall (bottom) compression for FEMNIST (left) and PersonaChat (right).}
\label{fig:femnist_gpt_all_complete}
\end{figure*}

\subsection{FEMNIST}
\label{appendix:femnist}
The dataset consists of 805,263 $28\times28$ pixel grayscale images distributed unevenly over 3,550 classes/users, with an average of 226.83 datapoints per user and standard deviation of 88.94. 
We further preprocess the data using the preprocessing script provided by the LEAF repository, using the command: \texttt{./preprocess.sh~-s~niid~--sf~1.0~-k~0~-t~sample}.
This results in 706,057 training samples and 80,182 validation samples over 3,500 clients.
\footnote{Leaf repository: https://tinyurl.com/u2w3twe}

We train a 40M-parameter ResNet101 with layer norm instead of batch norm, using an average batch size of $\approx 600$ (but varying depending on which clients participate) with standard data augmentation via image transformations and a triangular learning rate schedule.
When we train for 1 epoch, the pivot epoch of the learning rate schedule is 0.2, and the peak learning rate is 0.01.
When we train for fewer epochs in \fedavg{}, we compress the learning rate schedule accordingly.

For \fedssgd{} we grid-search values for $k$ and the number of columns.
For \fedssgd{} we search over $k$ in [50, 100, 200]~$\times 10^3$.
and the number of sketch columns in [1, 2, 5, 10]~$\times 10^6$.
For local top-$k$ we search over $k$ in [10, 20, 50, 100, 200, 500, 1000, 2000, 5000, 10000, 20000]~$\times 10^3$.
We do not use local momentum for local top-$k$, since each client only participates once.
For \fedavg{}, we search over the number of global epochs in [0.125, 0.1667, 0.25, 0.5], the number of local epochs in [1,2,5], and the local batch size in [10,20,50].
Figure \ref{fig:femnist_gpt_complete} shows the Pareto frontier of results for each method, broken down into upload, download, and overall compression.
Figure \ref{fig:femnist_gpt_all_complete} shows all results that converged.

\begin{figure}[H]
    \centering
    \includegraphics[width=0.4\textwidth]{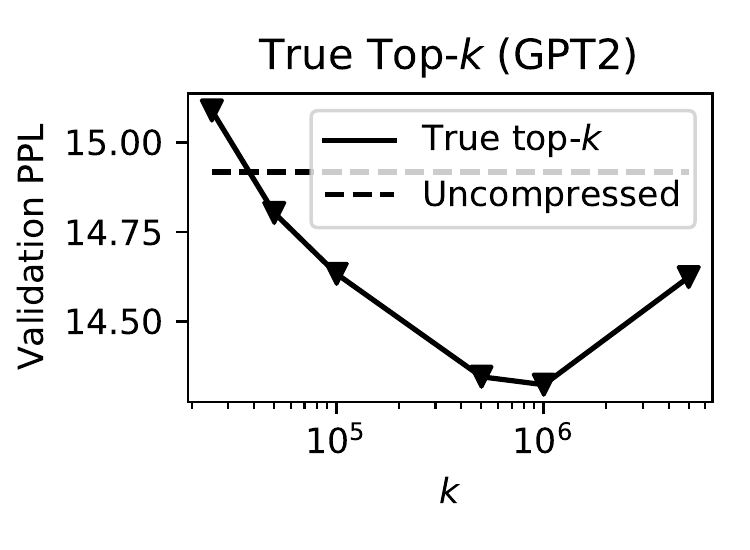}
    \caption{Validation perplexity on PersonaChat for a range of $k$ using true top-$k$. For $k \approx 10^{6}$, true top-$k$ provides some regularization, increasing performance over the uncompressed baseline. For larger $k$, the use of momentum factor masking degrades performance.}
    \label{fig:gpt2_truetopk}
\end{figure}

\subsection{PersonaChat}
\label{appendix:personachat}
The non-i.i.d. nature of PersonaChat comes from the fact that different Mechanical Turk workers were provided with different ``personalities,'' which are short snippets, written in English, containing a few salient characteristics of a fictional character.
We preprocess the dataset by creating additional tokens denoting the persona, context, and history, and feed these as input to a 124M-parameter GPT2 \cite{radford2019language} model created by HuggingFace \cite{wolf2019huggingface} based on the Generative Pretrained Transformer architecture proposed by OpenAI \cite{radford2019language}. 
We further augment the PersonaChat dataset by randomly shuffling the order of the personality sentences, doubling the size of the local datasets.

We use a linearly decaying learning rate of 0.16, with a total minibatch size of $\approx 64$ including the personality augmentation.
This can vary depending on which workers participate, as the local datasets are unbalanced.

\fedssgd{}, \fedavg{} and local top-$k$ each have unique hyperparameters which we need to search over.
For \fedssgd{} we try 6 points in a grid of values for $k$ and the number of columns.
For \fedssgd{}, we search over $k$ in [10, 25, 50, 100, 200]~$\times 10^3$, and over the number of sketch columns in [1240, 12400]~$\times 10^3$.
For local top-$k$, we search over $k$ in [50, 200, 1240, 5000]~$\times 10^3$.
For \fedavg{}, we search over the number of global epochs in [0.1, 0.2, 0.5] ($10\times$, $5\times$, and $2\times$ compression) and the number of local epochs in [2,5,10]. We always use the entire local dataset for each local iteration.

We report the perplexity, which is the average per word branching factor, a standard metric for language models.
Although we use the experimental setup and model from \citet{wolf2019huggingface}, our perplexities cannot be directly compared due to the modifications made to the choice of optimizer, learning rate, and dataset augmentation strategy.
Table \ref{tab:gpt2} shows perplexities, with standard deviations over three runs, for representative runs for each compression method.
Learning curves for these runs are shown in Figure \ref{fig:gpt2}.
Local top-$k$ consistently performs worse on this task when using global momentum (see Figure \ref{fig:gpt2}), so we only include results without momentum.
Local momentum is not possible, since each client participates only once.

Plots of perplexity vs. compression, broken down into upload, download, and overall compression, can be found in Figure \ref{fig:femnist_gpt_complete}.

We note that \fedssgd{} approximates an algorithm where clients send their full gradients, and the server sums those gradients but only updates the model with the $k$ highest-magnitude elements, saving the remaining elements in an error accumulation vector.
We explore this method, called true top-$k$, briefly in Figure \ref{fig:gpt2_truetopk}, which shows the method's performance as a function of $k$.
For intermediate values of $k$, true top-$k$ actually out-performs the uncompressed baseline, likely because it provides some regularization.
For large $k$, performance reduces because momentum factor masking inhibits momentum.

\begin{table}[H]
    \centering
    \begin{tabular}{r|rllll}
         Method              & $k$     &  PPL             & Download    & Upload      & Total \\
                             &         &                  & Compression & Compression & Compression \\
         \hline \\
         Uncompressed        &   --    & $14.9 \pm 0.02$  & $1\times$    & $1\times$    & $1\times$ \\
         Local Top-$k$       & 50,000  & $19.3 \pm 0.05$  & $30.3\times$ & $2490\times$ & $60\times$ \\
         Local Top-$k$       & 500,000 & $17.1 \pm 0.02$  & $3.6\times$  & $248\times$  & $7.1\times$ \\
         FedAvg (2 local iters) & --   & $16.3 \pm 0.2$   & $2\times$    & $2\times$    & $2\times$ \\
         FedAvg (5 local iters) & --   & $20.1 \pm 0.02$  & $5\times$    & $5\times$    & $5\times$ \\
         Sketch (1.24M cols) & 25,000  & \textbf{15.8} $\pm$ \textbf{0.007} & $3.8\times$  & $100\times$  & $7.3\times$ \\
         Sketch (12.4M cols) & 50,000  & \textbf{14.8} $\pm$ \textbf{0.002} & $2.4\times$  & $10\times$   & $3.9\times$ \\
    \end{tabular}
    \caption{Validation perplexities, with standard deviations measured over three different random seeds, for representative runs with \fedssgd{}, local top-$k$, and \fedavg{} on GPT2. Loss curves for these hyperparameter settings can be found in Figure \ref{fig:gpt2}.}
    \label{tab:gpt2}
\end{table}

\clearpage
\section{Theoretical properties}
\label{appendix:theory}

Theorems \ref{thm:convergence_assumption} and \ref{thm:main} both rely on the concept of $\ell_2^2$-heavy hitters; in Theorem \ref{thm:convergence_assumption}, heavy hitters appear in the quantity $\q^t=\eta(\rho \u^{t-1}+ \g^{t-1})+\e^{t-1}$, and in Theorem \ref{thm:main}, they appear in sums of consecutive gradients $\g^t$.
for $\g\in \mathbb{R}^d$, $\g_i$ is a $(\tau,\ell_2^2)$-heavy hitter  (or $\tau$-heavy) if $\g_i^2 \geq \tau \|\g\|^2$.
Given this definition, Assumption \ref{ass:contraction} can be rephrased as saying that at every timestep, $\q^t$ contains at least one $(\tau, \ell_2^2)$-heavy hitter.
And Assumption \ref{ass:sliding_heavy} can be rephrased to say that sums of consecutive $\g$ vectors contain $(\tau, \ell_2^2)$-heavy hitters, with all remaining values in the gradients drawn from mean-zero symmetric noise distributions.

With this definition in mind, the following two sections present proofs of Theorems \ref{thm:convergence_assumption} and \ref{thm:main}, respectively.

\subsection{Scenario 1: }
In Scenario 1, we assume that a contraction property holds during all of training (Assumption \ref{ass:contraction}).
To be consistent with our experimental evaluation, we show that FetchSGD converges (Theorem \ref{thm:convergence_assumption}) when using a vanilla Count Sketch for error accumulation, and when recovering the $k$ highest-magnitude elements from the error accumulation sketch instead of recovering only $\tau$-heavy hitters.

\begin{proof}[Proof of Theorem \ref{thm:convergence_assumption}]

We first verify that the stochastic gradients constructed are stochastic gradients with respect to the empirical mixture $\frac{1}{C}\sum_{j=1}^C\cP_i$, and we calculate its second moment bound. At a given iterate $\w$, we sample $B \subseteq [C], |B| = W$, a set of $W$ clients \emph{uniformly} from $C$ clients at every iteration, and compute $\g = \frac{1}{W}\sum_{i=1}^W \g_i$, where $\g_i$ are stochastic gradients with respect to the distribution $\cP_i$ on client $i$. This stochastic gradient is unbiased, as shown below.
$$\E{\g} = \hat{\mathbb{E}}\E{[\g | i]} = \frac{1}{W}\frac{1}{{C\choose \W}}{C-1\choose W-1} \sum_{i=1}^C \Eu{\cP_i}{\g_i} = \frac{1}{C}\sum_{i=1}^C \nabla f_i(\w).$$
The norm of the stochastic gradient is bounded: $$\E{\norm{\g}}^2 = \hat{\mathbb{E}}\E{\norm{\frac{1}{W}\sum_{i\in B,  |B|=W}\g_i \ \big\vert  \ B }^2} \leq \frac{1}{{C \choose W}}\frac{1}{W^2}W {C-1 \choose W-1} \sum_{i=1}^C \Eu{\cP_i}{\norm{\g_i}^2} 
\leq G^2$$

This proof follows the analysis of compressed SGD with error feedback in \citet{karimireddy2019error}, with additional momentum.  
Let $C(\x) = \text{Top-k}(\cU(\cS(\x)))$, the error accumulation then is $\cS(\e^{t+1}) = \cS(\eta(\rho \u^{t-1}+ \g^t)+\e^t)-\cS(C(\eta(\rho \u^{t-1}+ \g^t)+\e^t))$.
Consider the virtual sequence $\tilde \w^t = \w^t - \e^t - \frac{\eta \rho}{1-\rho}\u^{t-1}$. Upon expanding, we get 
\begin{align*}
\tilde \w^{t} &= \w^{t-1} - C(\eta(\rho\u^{t-2} + \g^{t-1})+\e^{t-1}) +  C(\eta(\rho\u^{t-2} + \g^{t-1})+\e^{t-1}) - \eta(\rho\u^{t-2} + \g^{t-1}) - \e^{t-1} - \frac{\eta \rho}{1-\rho}\u^{t-1}\\ 
& = \w^{t-1} - \e^{t-1} - \eta \g^{t-1} -  \eta \rho \u^{t-2} -  \frac{\eta \rho}{1-\rho}(\rho \u^{t-2}+\g^{t-1})\\
& = \w^{t-1} - \e^{t-1} - \eta\br{1+\frac{\rho}{1-\rho}} \g^{t-1} -  \eta \rho \br{1+\frac{\rho}{1-\rho}}\u^{t-2}\\
& = \w^{t-1}-\e^{t-1}-\frac{\eta \rho}{1-\rho}\u^{t-2} - \frac{\eta}{1-\rho}\g^{t-1} \\
& = \tilde \w^{t-1}-\frac{\eta}{1-\rho}\g^{t-1} 
\end{align*}
So this reduces to an SGD-like update but with a scaled learning rate. Applying $L$-smoothness of $f$, we get, 
\begin{align}
    \nonumber \E{f(\tilde \w^{t+1})} &\leq \mathbb{E}\left[f(\tilde \w^t) + \ip{\nabla f(\tilde \w^t)}{\tilde \w^{t+1} - \tilde \w^t} + \frac{L}{2}\norm{\tilde \w^{t+1} - \tilde \w^t}^2\right]
    \\&
      \nonumber\leq \E{f(\tilde \w^t)} - \frac{\eta}{(1-\rho)}\E{\ip{\nabla f(\tilde \w^t)}{\g^t}} + \frac{L\eta^2}{2(1-\rho)^2}\E{\norm{\g^t}^2}\\
      \nonumber &\leq \E{f(\tilde \w^t)} - \frac{\eta}{(1-\rho)}\E{\ip{\nabla f(\tilde \w^t)}{\nabla f(\w^t)}} + \frac{L\eta^2}{2(1-\rho)^2}\E{\norm{\g^t}^2} \\
    \nonumber    & \leq \E{f(\tilde \w^t)} - \frac{\eta}{(1-\rho)} \E{\norm{\nabla f(\w^t)}^2} + \frac{\eta}{2(1-\rho)}\br{\E{\norm{\nabla f(\w^t)}^2}+ \E{\norm{\nabla f(\tilde \w^t) - \nabla f(\w^t)}^2}} + \frac{L\eta^2G^2}{2(1-\rho)^2}\\
     \nonumber &\leq  \E{f(\tilde \w^t)} -  \frac{\eta}{2(1-\rho)}\E{\norm{\nabla f(\w^t)}^2}  + \frac{\eta L^2}{2(1-\rho)}\E{\norm{\tilde \w^t - \w^t}^2} + \frac{L\eta^2G^2}{2(1-\rho)^2}\\
      & =  \E{f(\tilde \w^t)} -  \frac{\eta}{2(1-\rho)}\E{\norm{\nabla f(\w^t)}^2}  + \frac{\eta L^2}{2(1-\rho)}\E{\norm{\e^t + \frac{\eta \rho}{1-\rho}\u^{t-1}}^2} + \frac{L\eta^2G^2}{2(1-\rho)^2}
      \label{eqn:plugin}
\end{align}

We now need to bound $\norm{\e^t+ \frac{\eta \rho}{1-\rho}\u^{t-1}}^2$.
However, we never compute or store $\e^t$ or $\u^t$, since the algorithm only maintains sketches of $\e^t$ and $\u^t$.
Instead, we will bound $\norm{\cS(\e^t)+ \frac{\eta \rho}{1-\rho}\cS(\u^{t-1})}^2$.
This is sufficient because $(1-\tau)\norm{\x} \leq \norm{\cS(\x)} \leq (1+\tau)\norm{\x}$, for a user-specified constant $\tau$ (which we will see later that it holds with high-probability due to the sketch size we use).
Note that  $\norm{S\left(\e^t+ \frac{\eta \rho}{1-\rho}\u^{t-1}\right)}^2 \leq 2\br{\norm{\cS(\e^t)}^2 + \br{\frac{\eta \rho}{1-\rho}}^2 \norm{\cS(\u^{t-1})}^2}$ because of linearity of sketching and the numerical inequality $(a+b)^2\leq 2(a^2+b^2)$.
We bound $\norm {\cS(\u^{t-1})}$ first:
\begin{align*} 
    \norm{\cS(\u^{t-1})}^2 =\norm{\sum_{i=1}^{t-1}\rho^i\cS(\g^i)}^2 \leq \br{\sum_{i=1}^{t-1} \rho^i\norm{\cS(\g^i)}}^2 \leq \br{\sum_{i=1}^{t-1} \rho^i (1+\tau)G}^2  \leq \br{\frac{(1+\tau)G}{1-\rho}}^2
\end{align*}
where the first inequality follows by application of triangle inequality for norms, and the second follows from $\norm{\cS(\x)} \leq (1+\tau)\norm{\x}$, and the bound on the gradients.
By definition of error accumulation, we have
\begin{align*}
    \norm{\cS(\e^t)}^2 &= \norm{\eta(\rho \cS(\u^{t-1})+ \cS(\g^{t-1}))+\cS(\e^{t-1})) - \cS(\text{Top-k}(\cU(\eta(\rho \cS(\u^{t-1})+ \cS(\g^{t-1}))+\cS(\e^{t-1}))))}^2 
\end{align*}
By Assumption \ref{ass:contraction}, $q^t = \eta(\rho \u^{t-1}+ \g^{t-1})+\e^{t-1}$ contains at least one $\tau$-heavy coordinate. All such coordinates will be successfully recovered by the unsketching procedure $\text{Top-k}(\cU(\cdot))$ with probability at least $1-\delta$ (depending on the size of Count Sketch, as discussed below), thus reducing the norm as follows: 
\begin{align*}    
    \norm{\cS(\e^t)}^2 &\leq (1-\tau)\norm{\eta(\rho \cS(\u^{t-1})+ \cS(\g^{t-1}))+\cS(\e^{t-1})}^2 
    \\&\leq (1-\tau)\br{(1+\gamma)\norm{\cS(\e^{t-1})}^2 + (1+1/\gamma)\eta^2\norm{\cS(\u^t)}^2} 
    \\&\leq (1-\tau)\br{(1+\gamma)\norm{\cS(\e^{t-1})}^2 + \frac{(1+1/\gamma)(1+\tau)^2\eta^2G^2}{(1-\rho)^2}} \\
    & \leq \sum_{i=0}^\infty \frac{(1+\tau)^2((1-\tau)(1+\gamma))^i(1+1/\gamma)\eta^2 G^2}{(1-\rho^2)} \\
    &\leq \frac{(1+\tau)^2(1-\tau)(1+1/\gamma)\eta^2 G^2}{1- ((1-\tau)(1+\gamma))}.
\end{align*}
where in the second inequality, we use the inequality $(a+b)^2 \leq (1+\gamma)a^2 + (1+1/\gamma)b^2$.
As argued in \citet{karimireddy2019error}, choosing $\gamma = \frac{\tau}{2(1-\tau)}$ suffices to upper bound the above with $\leq \frac{4(1+\tau)^2(1-\tau)\eta^2G^2}{\tau^2(1-\rho)^2}$.

Plugging everything into equation \ref{eqn:plugin}, we get that
\begin{align*}
     \E{\norm{\nabla f(\w^t)}^2} \leq \frac{2(1-\rho)}{\eta}\br{\E{f(\tilde \w^t)} -\E f(\tilde \w^{t+1})+ \frac{\eta L^2}{2(1-\rho)}\frac{4(1+\tau)^2\eta^2G^2}{(1-\tau)\tau^2(1-\rho)^2}+ \frac{L\eta^2G^2}{2(1-\rho)^2} }.
\end{align*}
Averaging over $T$ yields
\begin{align*}
 \min_{t=1\cdots T} \E{\norm{\nabla f(\w^t)}^2} \leq \frac{1}{T}\sum_{t=1}^T \E{\norm{\nabla f(\w^t)}^2} \leq \frac{2(1-\rho)(f(\w^0) -f^*)}{\eta T}+ \frac{4L^2(1+\tau)^2\eta^2G^2}{(1-\tau)\tau^2(1-\rho)^2}  + \frac{L\eta G^2}{(1-\rho)} .
\end{align*}

Setting $\eta = \frac{1-\rho}{2L\sqrt{T}}$ finishes the proof of convergence.

Now we will address the size of the sketch needed. As mentioned earlier, the sketch is required 1) to approximate the norm of $d$-dimensional vectors up to a multiplicative error of $(1\pm \tau)$, and 2) to recover all $\tau$-heavy coordinates. Following the Count Sketch memory complexity from~\citet{charikar2002finding}, we require memory of $\bigO{\frac{1}{\tau}\log{d/\delta}}$ to succeed with probability at least $1-\delta$. However, we reuse the same sketch over $T$ iterations, thus by a union bound we nee a sketch of size   $\bigO{\frac{1}{\tau}\log{dT/\delta}}$ to succeed with probability at least $1-\delta$. This completes the proof.

\end{proof}
Also, note that setting the momentum $\rho=0$ in the above, we recover a guarantee for \fedssgd{} with no momentum 
\begin{corollary}
Under the same assumptions as Theorem \ref{thm:convergence_assumption}, FetchSGD, with no momentum, in $T$ iterations,  outputs $\bc{\w^t}_{t=1}^T$ such that
\begin{align*}
    \min_{t=1 \cdots T}\E{\norm{\nabla f(\w^t)}^2} \leq \frac{4L(f(\w^0)-f^*) + G^2}{\sqrt{T}} + \frac{(1+\tau)^2G^2}{2(1-\tau)\tau^2 T}
\end{align*}
\end{corollary}
\subsection{Scenario 2}
In the previous section, we show convergence under Assumption \ref{ass:contraction}, which is relatively opaque and difficult to verify empirically.
In this section, we make the more interpretable Assumption \ref{ass:sliding_heavy}, which posits the existence of $\ell_2$-heavy hitters in the sequence of gradients encountered during optimization.
Under this assumption, \fedssgd{} is unlikely to converge when using a vanilla Count Sketch with error accumulation, by the following argument.
Under Assumption \ref{ass:sliding_heavy}, the useful signal in the sequence of gradients consists solely of $\ell_2$-heavy hitters spread over at most $I$ iterations.
As such, the norm of the signal at some iteration $t$ is bounded by $\bigO{I}$, whereas the norm of the error accumulation sketch overall (signal plus noise) is bounded by $\bigO{t}$, since the error accumulation includes the sum of gradient vectors up to time $t$.
Because noise hinders a Count Sketch's ability to recover heavy hitters, \fedssgd{} would have a difficult time converging once $t\gg I$.

To solve this problem, we show that \fedssgd{} converges when using a sliding window Count Sketch instead of a vanilla Count Sketch plus error accumulation.
Using a sliding window sketch solves the problem of noise growing as $\bigO{t}$ by recovering all of the signal present up until iteration $t-I$, and then discarding the remaining noise.
To see why this is the case, we consider a straightforward implementation of a sliding window Count Sketch that maintains $I$ individual Count Sketches $\{\S_e^i\}_{i=1}^I$, where the $i^{\text{th}}$ sketch was initialized at iteration $t-i$, as shown in Figure \ref{swpic2}.
On lines 12 and 14 of Algorithm \ref{alg:fedsketchedsgd}, we add a Count Sketch into $\S_e$ by simply adding the sketch to each of the $\S_e^i$.
On line 13, we recover heavy hitters ($\cU(\cdot)$) by unsketching each of the $\S_e^i$ and taking the union of the resulting heavy hitters.
And on line 16, we prepare the sliding window Count Sketch for the next iteration by setting $\S_e^{i+1}=\S_e^i$, and initializing $\S_e^0$ as an empty Count Sketch.

By constructing the sliding window data structure in this way, any sequence of up to $I$ gradients will appear in one of the $S_e^i$ at some iteration.
Therefore, a data structure of this sort will recover all $\ell_2$-heavy signal spread over up to $I$ iterations with probability $1-\delta$ when using individual sketches $\S_e^i$ of size $\bigO{\frac{1}{\tau}\log {\frac{dI}{\delta}}}$.
Because of this, when we discard $\S_e^i$ at the end of every iteration, we are only discarding noise, thereby preventing the noise from growing as $\bigO{t}$ without losing any useful signal.

We use the sliding window Count Sketch data structure described above to show convergence, but in Appendix \ref{appendix:sliding_window} we discuss more efficient implementations that require maintaining only $\log I$ instead of $I$ individual Count Sketch data structures, as depicted in \ref{swpic3}.
In addition, to simplify the presentation, instead of recovering the highest-magnitude $k$ elements from the sliding window error accumulation sketch, we recover only $\tau$-heavy hitters.

\begin{proof}[Proof of Theorem \ref{thm:main}] 
For clarity, we break the proof into two parts.
First, we address the particular case when $I=1$, and then we extend the proof to general $I$. 

\paragraph{Warm-up: $I=1$ (without error accumulation).}
When $I=1$, Assumption \ref{ass:sliding_heavy} guarantees that every gradient contains heavy hitters.
And the sliding window error accumulation sketch used by \fedssgd{} reduces to a simple Count Sketch for compression, with no error accumulation across iterations. In this case, the gradient update step is of the form
\begin{align*}
    &\w^{t+1} = \w^t - \cC(\eta \g^t).\\
\end{align*}
where $\cC(\cdot)$ is \text{Top-}$\tau(\cU(\cS(\cdot)))$. Consider the virtual sequence $\tilde \w^{t} = \w^t - \sum_{i=1}^{t-1}\br{\eta \g^i - \cC(\eta \g^i)}$. Upon expanding, we get
\begin{align*}
    \tilde \w^t = \w^{t-1} - \cC(\eta \g^{t-1}) - \sum_{i=1}^{t-1}\br{\eta \g^i - \cC(\eta \g^i)} = \w^{t-1} -  \sum_{i=1}^{t-2}\br{\eta \g^i - C(\eta \g^i)} -\eta \g^{t-1} = \tilde \w^{t-1} - \eta \g^{t-1}
\end{align*}

From $L$-smoothness of $f$, 
\begin{align}
    \nonumber \E{f(\tilde \w^{t+1})} 
    \nonumber&\leq \mathbb{E}f(\tilde \w^t) + \E{\ip{\nabla f(\tilde \w^t)}{\tilde \w^{t+1} - \tilde \w^t}} + \frac{L}{2}\E{\norm{\tilde \w^{t+1} - \tilde \w^t}^2}\\
     \nonumber&=\mathbb{E}f(\tilde \w^t) - \E{\eta\ip{\nabla f(\tilde \w^t)}{\g^t}} + \frac{L\eta^2}{2}\E{\norm{\g^t}^2}\\
   \nonumber  &\leq \mathbb{E}f(\tilde \w^t) - \mathbb{E}\eta\ip{\nabla f(\tilde \w^t)-\nabla f( \w^t)+\nabla f( \w^t)}{\nabla f( \w^t)} + \frac{\eta^2LG^2}{2}\\
     \nonumber  & = \mathbb{E}f(\tilde \w^t) - \eta \mathbb{E}\norm{\nabla f(\w^t)}^2 - \mathbb{E}\eta\ip{\nabla f(\tilde \w^t)-\nabla f( \w^t)}{\nabla f( \w^t)} + \frac{\eta^2LG^2}{2} \\
   \nonumber  & \leq \mathbb{E}f(\tilde \w^t) - \eta \mathbb{E}\norm{\nabla f(\w^t)}^2 + \frac{\eta}{2}\mathbb{E}\br{\norm{\nabla f(\tilde \w^t) - \nabla f(\w^t)}^2 + \norm{\nabla f(\w^t)}^2} + \frac{\eta^2LG^2}{2} \\
    \label{eqn:plugin2} & \leq \mathbb{E}f(\tilde \w^t) -  \frac{\eta}{2}\mathbb{E}\norm{\nabla f(\w^t)}^2  + \frac{\eta L^2}{2}\E{\norm{\tilde \w^t - \w^t}^2} + \frac{\eta^2LG^2}{2}
\end{align}
where in the third inequality, we used $\abs{\ip{\u}{\v}} \leq \frac{1}{2} \br{\norm{\u}^2 + \norm{\v}^2}$, and the last inequality follows from $L$-smoothness.

Now, to show convergence of $||\nabla f(\w^t)||$, we need to upper bound $\norm{\tilde \w^t - \w^t} = \norm{\sum_{i=1}^{t-1}(\cC(\eta\g^{i})- \eta \g^{i})}$.
To do so, we note that, conditioned on successful recovery of heavy hitters from the Count Sketch, $\cC(\eta \g^i)-\eta \g^i$ consists only of mean-zero symmetric noise:
every gradient $\g^i$ in a $(1, \tau)$-sliding heavy sequence of gradients consists solely of $\tau$-heavy hitters and mean-zero symmetric noise, by Definition \ref{defn:sliding_heavy}.
Therefore, when all $\tau$-heavy coordinates are identified, $\cC(\g^i)-\g^i=:\z^i$ consists only of $\g^i_N$ (from Definition \ref{defn:sliding_heavy}) and the Count Sketch heavy-hitter estimation error.
By Assumption \ref{ass:sliding_heavy}, $\g^i_N=:\z^i_{\text{noise}}$ is drawn from a mean-zero symmetric distribution with scale $\norm{\g^i}$.
And by the properties of the Count Sketch, the heavy-hitter estimation error $\z^i_{\text{estimation}}$ is as well.
Therefore, $\z^i =\Vert\g^i\Vert\xi^i$ for some $\xi^i$ drawn from mean-zero symmetric noise distributions, such that the 
$\xi^i$'s are mutually independent and independent of $\norm{\g^i}$.
As a result:

\begin{align*}
    \norm{\tilde \w^t - \w^t}  = \norm{\sum_{i=1}^{t-1}(\cC(\eta\g^{i})- \eta\g^{i})}=  \eta\norm{\sum_{i=1}^{t-1}\z^i_{\text{estimation}} +\z^i_{\text{noise}}} =\eta\norm{\sum_{i=1}^{t-1} \z^i} =\eta\norm{\sum_{i=1}^{t-1}\norm{\g^i} \xi^i}.
\end{align*}

Note that, since the $\g^i$'s are dependent because they are a sequence of SGD updates, the $\z^i$'s are also dependent. However since the $\xi^i$'s are independent with mean zero, $\E{\left[\norm{\g^i}\xi^i | \cF_i\right]} = 0$, where $\cF_i$ is the filtration of events before the $i^{\text{th}}$ iteration.
So the stochastic process $\bc{\z^i}_{i=1}^t$ forms a martingale difference sequence.
For any martingale difference sequence $\bc{\x^i}_{i=1}^T$, it holds that
\begin{align*}
\E{\norm{\sum_{i=1}^T \x^i -  \E{\x^i}}^2}  = \sum_{i=1}^T \E{\norm{\x^i - \E{\x^i}}}^2 + \sum_{i,j=1,i\neq j}^{T,T} \E{\ip{\x^i - \E{\x^i}}{\x^j - \E{\x^j}}}.
\end{align*}
For $i>j$, $\E{\ip{\x^i - \E{\x^i}}{\x^j - \E{\x^j}}} = \mathbb{E}^j\left[\E{\ip{\x^i - \E{\x^i}}{\x^j - \E{\x^j}}}|j\right] = 0$. Applying this, we get
\begin{align*}
    \E{\norm{\sum_{i=1}^{t-1}(\cC(\eta\g^i)- \eta \g^i}}^2 =  \eta^2 \E{\norm{\sum_{i=1}^{t-1}\norm{\g^i}\xi^i}}^2 = \eta^2 \sum_{i=1}^{t-1} \E{\norm{\norm{\g^i}\xi^i}^2}  = \eta^2\sum_{i=1}^{t-1}\E{\norm{\z^i}^2},
\end{align*}
where in the second equality, we apply the martingale difference result, noting that the random variables $\z^i = \norm{\g^i}\xi^i$ have a mean of zero.
We will now look how heavy coordinates and noise coordinates contribute to the norm of $\z^i$.
We can decompose $\norm{\z^i}^2 = \norm{\z^i_\text{estimation} + \z^i_{\text{noise}}}^2 \leq 2(\norm{\z^i_\text{estimation}}^2 +  \norm{\z^i_{\text{noise}}}^2)$.
Now we bound each component of $\z^i$: $\z^i_{\text{estimation}}$ and $\z^i_{\text{noise}}$.

From Lemma 2 in \citet{charikar2002finding}, for each bucket in the Count Sketch, the variance in estimation is at most the $\ell_2$ norm of the tail divided by the number of buckets $b$.
Since the tail has mass at most $(1-\tau)G^2$, for each coordinate $j$, we have $\E{(\z_\text{estimation}^i)_j^2} \leq \frac{(1-\tau)G^2}{b}$.
There are at most $\frac{1}{\tau^2}$ heavy coordinates present, and based on the sketch size in Theorem \ref{thm:main}, we can choose a sketch that has at least $b > \frac{1}{\tau^2}$ buckets, so $\E{\norm{\z^i_\text{estimation}}^2} \leq \frac{1} {\tau^2} \cdot\frac{(1-\tau)G^2}{b} \leq (1-\tau)G^2 $.
Also, from Assumption \ref{ass:sliding_heavy}, $\E{\norm{\z^i_\text{noise}}^2} \leq \beta G^2$.%

As in the proof of Theorem \ref{thm:convergence_assumption}, plugging this in equation \ref{eqn:plugin}, taking $\norm{\nabla f(\w^t)}$ to the left hand side, averaging and taking expectation with respect to all randomness, we get that

\begin{align*}
 \min_{t=1\cdots T} \E{\norm{\nabla f(\w^t)}^2} \leq \frac{\sum_{i=1}^T \E{\norm{\nabla f(\w^i)}^2}}{T} \leq \frac{(f(\w^0) -f^*)}{\eta T}+  \eta^2 2(1-\tau +\beta) G^2 L T +  L\eta G^2.
\end{align*}

Finally, choosing $\eta = \frac{1}{G\sqrt{L}T^{2/3}}$, we get that
\begin{align*}
 \min_{t=1\cdots T} \E{\norm{\nabla f(\w^t)}^2} \leq \frac{G\sqrt{L}\br{(f(\w^0) -f^*)}  + 2(1-\tau+\beta)}{T^{1/3}} +\frac{G \sqrt{L}}{T^{2/3}}.
\end{align*}

Note that the analysis above holds conditioned on the success of the Count Sketch data structure and on the event that the first statement in Definition~\ref{defn:sliding_heavy} holds at every iteration, which happens with probability 
$1 - 2T\delta $ by a union bound. This leads to the size of the sketch provided in Theorem \ref{thm:main}.

\paragraph{General case: any $I$ (with error accumulation)}
\label{appendix:proof}

\begin{figure*}[t]
\begin{subfigure}[b]{0.49\textwidth}
\centering
\includegraphics[width=170px]{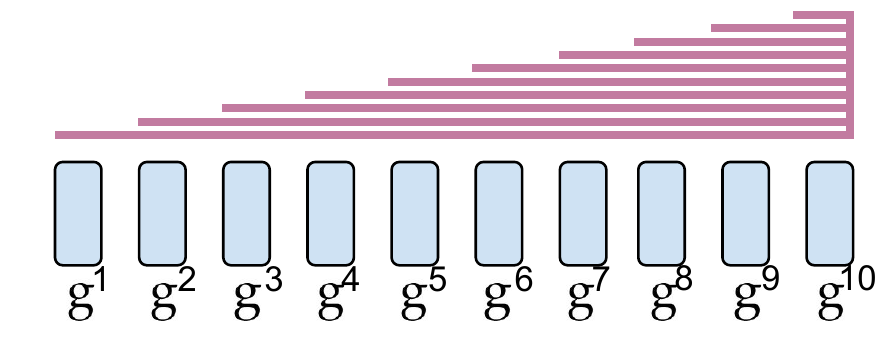}
\caption{Naive sliding window}
\label{swpic2}
\end{subfigure}
\begin{subfigure}[b]{0.49\textwidth}
\centering
\includegraphics[width=170px]{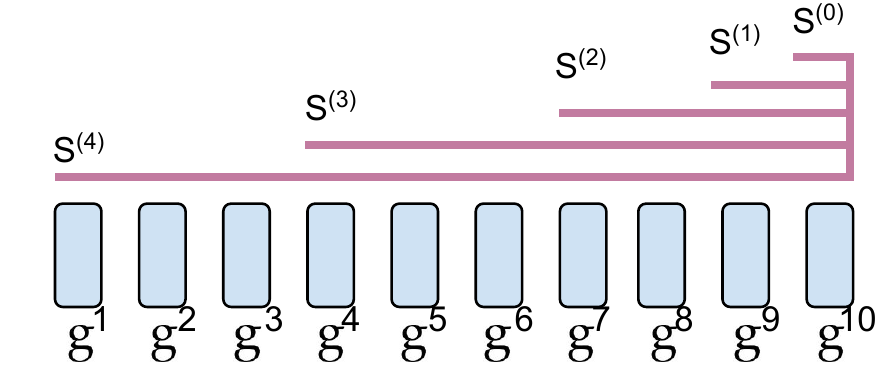}
\caption{Adaptive sliding window}
\label{swpic3}
\end{subfigure}
\caption{Approximating $\ell_2$ norm on sliding windows.}
\end{figure*}

When $I=1$, \fedssgd{} converges when each gradient is compressed with a vanilla Count Sketch, since doing so preserves all $\ell_2$-heavy signal in the sequence of gradients.
For $I>1$, this is no longer the case, since signal needed for convergence may not be $\ell_2$-heavy in any given iteration.
As described above, we capture signal spread over multiple iterations using a sliding-window Count Sketch data structure, which recovers all the $\ell_2$-heavy signal with high probability, even if it is spread over multiple iterations.
This is sufficient to show convergence because Assumption \ref{ass:sliding_heavy} states that all of the signal in the sequence of gradients is contained within gradient coordinates that are $\ell_2$-heavy over a sliding window.

Because of the similarity between a Count Sketch and a sliding window Count Sketch, the proof of Theorem \ref{thm:main} for general $I$ largely follows the proof for $I=1$.

Let $C^t=\cU(\S_e^t)$ be the result of unsketching the sliding window Count Sketch at iteration $t$, and consider a virtual sequence similar to the one in the warm-up case:
\begin{align*}
    \tilde \w^t &= \w^t - \left(\sum_{i=1}^{t-1} \eta \g^i - C^i\right)\\
    & = \w^{t-1} - C^{t-1} -\left(\sum_{i=1}^{t-1} \eta \g^i - C^i\right)\\ 
    & = \w^{t-1} - \sum_{i=1}^{t-2} \eta \g^i - C^i - \eta \g^{t-1} \\ 
    & = \tilde \w^{t-1} - \eta \g^{t-1}
\end{align*}
 
As before, we need to bound
$\norm{\tilde\w^t - \w^t} = \norm{\sum_{i=1}^{t-1} \eta \g^i - C^i}$.

Because the sliding window sketch recovers all $\ell_2$-heavy signal spread over at most $I$ iterations, the value of $\eta \g^i - C^i$ consists of only $z_{\text{estimation}}+z_{\text{noise}}$ when $i<t-I$.

For $i < t - I$, the sliding window Count Sketch data structure will recover all the signal, leaving only $z^i_\text{estimation}+z^i_\text{noise}$ remaining.
For $t-I \leq i \leq t$, some signal will already be recovered in $C^i$ (which we denote $\g^i_r$), while other signal remains to be recovered in future steps ($\g^i_n$).
Note that $\g^i=\g^i_r+\g^i_n$, and we let $z_{\text{noise}}$ be distributed arbitrarily between $\g^i_r$ and $\g^i_n$.

As shown in the warm-up case, we argue that

\begin{align*}
     \tilde\w^t - \w^t &= \sum_{i=1}^t (\eta \g^i - \C^i) =  \left(\sum_{i=1}^{t-I} \eta \g^i +  \sum_{i=t-I+1}^{t} \eta \g^i_r - \sum_{i=1}^t C^i\right) + 
     \sum_{i=t-I+1}^{t}\eta \g^i_n \\
     &=   \left(\sum_{i=1}^{t} z_\text{estimation error}^i + z_\text{noise}^i \right)  + \sum_{i=t-I+1}^{t}\eta \g^i_n = \sum_{i=1}^{t} \z^i + \sum_{i=t-I+1}^{t}\eta \g^i_n
\end{align*}

Since the gradients are bounded in norm, the norm of the sum of the past $I$ gradients, from which signal has yet to be recovered, can be bounded as $IG$.
The norm of $\g^i_n$ is less than the norm of $\g^i$, so the sum of $\g^i_n$ can be likewise bounded.

Then, by the triangle inequality we get 
$$\norm{\tilde\w^t - \w^t}^2 \leq 2\norm{\sum_{i=1}^{t-I} \z^i}^2  +2 \eta^2 I^2G^2$$
We now similarly argue that $\z^i$ forms a martingale difference sequence and therefore we have

\begin{align*}
    \E{\norm{\tilde \w^t - \w^t}}^2 \leq 2\E{\norm{\sum_{i=1}^{t-I} \z^i}} +2 \eta^2 I^2G^2 \leq   2(1-\tau+\beta)\eta^2G^2(t-I) +2 \eta^2 I^2G^2 \leq  2(1-\tau+\beta)\eta^2G^2t +2 \eta^2 I^2G^2
\end{align*}
Repeating the steps in the warm-up case: using $L$-smoothness of $f$, we get
\begin{align*}
    \E{f(\tilde \w^{t+1})} & \leq f(\tilde \w^t) -  \frac{\eta}{2}\norm{\nabla f(\w^t)}^2  + \frac{\eta L}{2}\E{\norm{\tilde \w^t - \w^t}^2} + \frac{\eta^2LG^2}{2} \\
    &\leq f(\tilde \w^t) -  \frac{\eta}{2}\norm{\nabla f(\w^t)}^2  + \frac{\eta L}{2}\br{2(1-\tau+\beta)\eta^2G^2t +2 \eta^2 I^2G^2} + \frac{\eta^2LG^2}{2}
\end{align*}

Taking $\norm{\nabla f(\w^t)}$ to the left hand side, averaging and taking expectation with respect to all randomness, and choosing $\eta = \frac{1}{G\sqrt{L}T^{2/3}}$  we get

\begin{align*}
 \min_{t=1\cdots T} \E{\norm{\nabla f(\w^t)}^2}& \leq 
 \frac{ G\sqrt{L}\br{f(\w^0) -f^*}  + 2(1-\tau+\beta)}{T^{1/3}} +\frac{G \sqrt{L}}{T^{2/3}} + \frac{2I^2}{T^{4/3}}\\
\end{align*}

The first part of the theorem is recovered by noting that $\beta\leq 1$. For the second part, note that the size of sketch needed to capture $\tau$-heavy hitters with probability at least $1-\delta$ is $\bigO{\frac{\log{d\delta}}{\tau^2}}$; taking a union bound over all $T$ iterations recovers the second claim in the theorem.

\end{proof}

\paragraph{Implementation.}
We now give details on how this data structure is constructed and what the operations correspond to.
For all heavy coordinates to be successfully recovered from all suffixes of the last $I$ gradient updates (i.e. $\forall I' < I$, to recover heavy coordinates of $\sum_{i=t-I'}^t \eta\g^i$) we can maintain $I$ sketches in the overlapping manner depicted in Figure~\ref{swpic2}.
That is, every sketch is cleared every $I$ iterations.
To find heavy coordinates, the FindHeavy() method must query every sketch and return the united set of heavy coordinates found; Insert() appends new gradients to all $I$ sketches; and Update() subtracts the input set of heavy coordinates from all $I$ sketches.
Although sketches are computationally efficient and use memory sub-linear in $d$ (a Count Sketch stores $\bigO{\log d}$ entries), linear dependency on $I$ in unfavorable, as it limits our choice of $I$.
Fortunately, the sliding window model, which is very close to the setting studied here, is thoroughly studied in the streaming community~\cite{braverman2007smooth,datar2002maintaining}.
These methods allow us to maintain a number of sketches only logarithmic in $I$.
For a high level overview we refer the reader to Appendix~\ref{appendix:sliding_window}.

\subsection{Are these assumptions necessary?} We have discussed that un-sketching a sketch gives an unbiased estimate of the gradient: $\E{\cU(\cS(\g))} = \g$, so the sketch can be viewed as a stochastic gradient estimate. Moreover, since Top-$k$, error feedback and momentum operate on these new stochastic gradients, existing analysis can show that our method converges. However, the variance of the estimate derived from unsketching is $\Theta(d)$, in the worst-case. By standard SGD analysis, this gives a convergence rate of $\bigO{d/\sqrt{T}}$, which is optimal since the model is a function of only these new $\bigO{d}$-variance stochastic gradients.  This establishes that even without any assumptions on the sequence of gradients encountered during optimization, our algorithm has convergence properties. However this dimensionality dependence does not reflect our observation that the algorithm performs competitively with uncompressed SGD in practice, motivating our assumptions and analysis.

\section{Count Sketch}
\label{appendix:countsketch}

Streaming algorithms have aided the handling of enormous data flows for more than two decades.
These algorithms operate on sequential data updates, and their memory consumption is sub-linear in the problem size (length of stream and universe size).
First formalized by \citet{alon1999space}, sketching (a term often used for streaming data structures) facilitates numerous applications, from handling networking traffic \cite{ivkin2019qpipe} to analyzing cosmology simulations \cite {liu2015streaming}.
In this section we provide a high-level overview of the streaming model, and we explain the intuition behind the Count Sketch \cite{charikar2002finding} data structure, which we use in our main result.
For more details on the field, we refer readers to \citet{muthukrishnan2005data}.

Consider a frequency vector $g\in \R^d$ initialized with zeros and updated coordinate by coordinate in the streaming fashion -- i.e. at time $t$ an update $(a_i, w_i)$ changes the frequency as $g_{a_i} += w_i$.
\citet{alon1999space} introduce the AMS sketch, which can approximate $\norm{g}$ with only constant memory.
Memory footprint is very important in the streaming setting, since $d$ is usually assumed to be large enough that $g$ cannot fit in the memory.
The AMS sketch consists of a running sum $S$ initialized with $0$, and a hash function $h$ that maps coordinates of $g$ into $\pm1$ in an i.i.d. manner.
Upon arrival of an update $(a_i, w_i)$, the AMS sketch performs a running sum update: $S$ += $h(a_i)w_i$. 
Note that at the end of the stream, $\E(S) = \sum_{i=1}^{n}{h(a_i)w_i}$ can be reorganized as per coordinate 
$\E(S) = \sum_{j=1}^{d}{\left(h(j)\sum_{\left\{ i:a_i=j \right\}}{w_i}\right)} = \sum_{j=1}^{d}{h(j)g_j}$, where $g_j$ is the value of $j$-th coordinate  at the end of the stream.
The AMS sketch returns $S^2$ as an estimation of $\|g\|^2$: $\E(S^2) = \E(\sum_{j=1}^{d}{h(j)^2g_j^2}) + \E(\sum_{j=1}^{d}{h(j)h(j')g_jg_{j'}})$. 
If $h$ is at least $2$-wise independent second, then both $\E h(j)h(j')$ and the second term are $0$. %
So $\E(S^2) = \E(\sum_{j=1}^{d}{g_j^2}) = \|g\|^2$, as desired.
Similarly, \citet{alon1999space} show how to bound the variance of the estimator (at the cost of $4$-wise hash independence). 
The AMS sketch maintains a group of basic sketches described above, so that the variance and failure probability can be controlled directly via the amount of memory allocated: an AMS sketch finds $\hat \ell_2 = \|g\| \pm \varepsilon \|g\|$ using $O(1/\varepsilon^2)$ memory.   

The Count Sketch data structure \cite{charikar2002finding} extends this technique to find heavy coordinates of the vector.
A coordinate $i$ is $(\tau, \ell_2)$-heavy (or an $(\tau, \ell_2)$-heavy hitter) if $g_i \ge \tau \|g\|$.
The intuition behind the Count Sketch is as follows:
the data structure maintains a hash table of size $c$, where every coordinate $j\in[d]$ is mapped to one of the bins, in which an AMS-style running sum is maintained.
By definition, the heavy coordinates encompass a large portion of the $\ell_2$ mass, so the $\ell_2$ norm of the bins where heavy coordinates are mapped to will be significantly larger then that of the rest of the bins.
Consequently, coordinates mapped to the bins with small $\ell_2$ norm are not heavy, and can be excluded from list of heavy candidates.
Repeating the procedure $O(\log d)$  times in parallel reveals the identities of heavy coordinates and estimates their values.
Formally, a Count Sketch finds all $(\tau, \ell_2)$-heavy coordinates and approximates their values with $\pm \varepsilon \|g\|$ additive error.
It requires $O(\frac{1}{\varepsilon^2\tau^2}\log d)$ memory.
Algorithm~\ref{code:cs} depicts the most important steps in a Count Sketch.
For more details on the proof and implementation, refer to \cite{charikar2002finding}. 
\begin{algorithm}[ht]
    \caption{Count Sketch \citep{charikar2002finding}}
    \label{code:cs}
    \begin{algorithmic}[1]
        \small
        \STATE \textbf{function} init($r$, $c$):
        \STATE ~~~~ init $r\times c$ table of counters $S$
        \STATE ~~~~ for each row $r$ init sign and bucket hashes:  $\left\{ (h^s_j, h^b_j)\right\}_{j=1}^r$
        
        \STATE \textbf{function} update($(a_i, w_i)$):
        \STATE ~~~~ \textbf{for} $j$ in $1\ldots r:\;\;\;S[j, h^b_j(i)]$ += $h^s_j(i)w_i$
    
        \STATE \textbf{function} estimate($i$):
        \STATE ~~~~ init length $r$ array estimates
        \STATE ~~~~ \textbf{for} $j$ in $1,\ldots, r$:
        \STATE ~~~~ ~~~~ estimates$[r] = h^s_j(i) S[j, h^b_j(i)]$
        \STATE ~~~~ \textbf{return} median(estimates)
        
    \end{algorithmic}
\end{algorithm}

For \fedssgd{}, an important feature of the Count Sketch data structure is that it is linear -- \textit{i.e.}, $\cS(\g_1)+ \cS(\g_2) = \cS(\g_1 + \g_2)$. This property is used when combining the sketches of gradients computed on every iteration, and to maintain error accumulation and momentum.  %
We emphasize that while there are more efficient algorithms for finding heavy hitters, they either provide weaker $\ell_1$ approximation guarantees \cite{muthukrishnan2005data} or support only non-negative entries of the vector \cite{misra1982finding, braverman2017bptree}.
The structure of the Count Sketch allows for high amounts of parallelization, and the operations of a Count Sketch can be easily accelerated using GPUs \cite{ivkin2018scalable}.  

\section{Sliding Windows}
\label{appendix:sliding_window}
As was mentioned in Appendix~\ref{appendix:countsketch}, the streaming model focuses on problems where data items arrive sequentially and their volume is too large to store on disk.
In this case, accessing previous updates is prohibited, unless they are stored in the sketch.
In many cases, the stream is assumed to be infinite and the ultimate goal is to approximate some function on the last $n$ updates and to \emph{``forget''} the older ones.
The sliding window model, introduced in \cite{datar2002maintaining}, addresses exactly this setting.
Recall the example from Appendix~\ref{appendix:countsketch}: given a stream of updates $(a^t, w^t)$ to a frequency vector $g$ (i.e. $g^t_{a^t} += w^t$), approximating the $\ell_2$ norm of $\g$ in the streaming model implies finding $\hat\ell_2 = \|g\| \pm \varepsilon \|g\|$
On the other hand, in the sliding window model one is interested only in the last $n$ updates, i.e. $\hat\ell_2 = \|g^t - g^{t-n}\| \pm \varepsilon \|g^t - g^{t-n}\|$. 

One naive solution is to maintain $n$ overlapping sketches, as in Fig.~\ref{swpic2}.
However, such a solution is infeasible for larger~$n$.
Currently there are $2$ major frameworks to \emph{adopt} streaming sketches to the sliding window model: exponential histograms, by~\citet{datar2002maintaining}, and smooth histograms, by~\citet{braverman2007smooth}.
For simplicity, we will provide only the high level intuition behind the latter one.
Maintaining all $n$ sketches as in Fig.~\ref{swpic2} is unnecessary if one can control the growth of the function: neighboring sketches differ only by one gradient update, and the majority of the sketches can be pruned.
\citet{braverman2007smooth} show that if a function is monotonic and satisfies a smoothness property, then the sketches can be efficiently pruned, leaving only $\bigO{\log n}$ sketches. As in Fig.~\ref{swpic3}, $\|S^{(i)}\| < (1+\varepsilon)\|S^{(i-1)}\|$, so any value in the intermediate suffixes (which were pruned earlier) can be approximated by the closest sketch $\|S^{(i)}\|$ as shown in~\citet{ivkin2019know}.
For more details on how to construct this data structure, and for a definition of the smoothness property, we refer readers to \citet{braverman2007smooth}. 

\end{document}